\documentclass{article}

\usepackage{arxiv}

\usepackage[utf8]{inputenc} %
\usepackage[T1]{fontenc}    %
\usepackage{hyperref}       %
\usepackage{url}            %
\usepackage{booktabs}       %
\usepackage{amsfonts}       %
\usepackage{nicefrac}       %
\usepackage{microtype}      %
\usepackage{xcolor}         %

\usepackage[numbers]{natbib}
\usepackage{multicol}
\usepackage{booktabs}       %
\usepackage{lipsum}		%
\usepackage{graphicx}
\usepackage{doi}
\usepackage{amsfonts,amsmath,amssymb,amsthm}
\usepackage{cleveref}
\usepackage{enumitem}
\usepackage{xcolor}
\usepackage{wrapfig}

\usepackage{thmtools}

\usepackage{subcaption}

\newtheorem{theorem}{Theorem}
\newtheorem{lemma}[theorem]{Lemma}

\newtheorem{corollary}[theorem]{Corollary}
\newtheorem{definition}[theorem]{Definition}
\newtheorem{assumption}[theorem]{Assumption}

\newcommand{\E}{\mathbb{E}}

\newcommand{\R}{\mathbb{R}}

\newcommand{\A}{\mathcal{A}}
\newcommand{\Z}{\mathcal{Z}}
\newcommand{\M}{\mathcal{M}}
\newcommand{\G}{\mathcal{G}} %
\newcommand{\Sspace}{\mathcal{S}}
\newcommand{\Eenv}{\mathcal{E}}
\newcommand{\Csafe}{\mathcal{C}_{\mathrm{safe}}}

\newcommand{\Zspace}{\mathcal{Z}}

\newcommand{\Vol}{\mathrm{Vol}}
\newcommand{\esssup}{\operatorname*{ess\,sup}}

\newcommand{\mypara}[1]{\medskip\noindent\textbf{#1.}~}

\title{Action Hallucination in Generative Vision-Language-Action Models}

\author{%
  Harold Soh \,\,\textnormal{and}\,\, Eugene Lim \\
  \textsuperscript{1}Department of Computer Science, School of Computing \\
  \textsuperscript{2}Smart Systems Institute \\
  National University of Singapore \\
  \href{mailto:harold@nus.edu.sg}{\texttt{harold@nus.edu.sg}}
  \quad
  \href{mailto:elimwj@u.nus.edu}{\texttt{elimwj@u.nus.edu}}
}

\begin{document}

\maketitle

\begin{abstract}
Robot Foundation Models, such as VLAs, promise end-to-end generative robot policies with broad generalization. Yet it remains unclear whether they fundamentally resolve the core problem of action generation in embodied settings, or overcome the long-standing challenges of robotics. We address this question by analyzing action hallucinations that violate physical constraints and their extension to plan-level failures. Focusing on latent-variable generative policies, we show that hallucinations can arise from structural mismatches between feasible robot behavior and common model architectures. We study three such barriers---topological, precision, and horizon---and show how they impose unavoidable tradeoffs. Our analysis  provides mechanistic explanations for  reported empirical failures of generative robot policies and suggests principled directions for improving reliability and trustworthiness, without abandoning their expressive power.
\end{abstract}

\section{Introduction}
\label{sec:intro}

Robot Foundation Models (RFMs) and Large-scale Vision-Language-Action (VLA) models~(e.g.,~\cite{rt-2,openvla,black2024pi_0,nvidia2025gr00tn1openfoundation}) promise a ``GPT moment'' for robotics: given image observations and a natural-language instruction, a single end-to-end policy produces a continuous control trajectory that solves the task.
Recent generative VLAs such as $\pi_{0.5}$~\cite{intelligence2025pi}, Gr00T N1~\cite{nvidia2025gr00tn1openfoundation}, and MolmoAct~\cite{molmoact} combine transformer encoders with conditional flow/diffusion models that map Gaussian noise into full action trajectories. Other contenders like diffusion policies~\cite{chi2023diffusionpolicy,liu2025rdtb} and flow-matching policies~\cite{zhang2024flowpolicy} map Gaussian noise into action chunks that are executed sequentially. These models demonstrate impressive semantic generalization, yet substantial prior work~\cite{zhai2025vfp,dai2025safe,yang2025embodiedbench,jia2024towards,dai2025}  has shown they also generate actions that are physically invalid, e.g., grasps through objects or plans that do not result in goal achievement. We refer to this phenomenon as \emph{action hallucination}.

Action hallucinations matter for two reasons. First, they are safety- and reliability-critical. Unlike hallucinated text, invalid actions are executed in the physical world and can cause damage  or hazardous interactions. Second, these failures may not be confined to exotic corner cases. 
Even a model that is highly capable \emph{on average} can intermittently produce nonsensical actions, a risk that is amplified when robots are deployed at scale and over long periods of time.

This paper argues that part of action hallucination is \emph{structural}: it arises from a mismatch between (i) the geometry/topology of physically valid behavior in robotics and (ii) the architectural regularities of modern generative action heads and planning pipelines. Concretely, many VLAs generate actions by sampling a \emph{connected} latent prior (typically Gaussian noise) and decoding it through a map that is continuous in the latent (as in diffusion, flow-matching, and conditional flows). At execution time, these models are often used sequentially and/or wrapped in test-time sampling (and possible verification). This work clarifies \emph{when} these design choices induce hallucinations (regardless of the data or model scale) and \emph{what} changes are needed to mitigate them.

\mypara{Contributions} We use largely elementary  mathematical tools (topology, measure bounds, and probabilistic arguments) to construct a coherent theory of {action hallucination} that matches the architectural and algorithmic characteristics of modern VLAs. 
This approach connects to classic insights in robotics and leads to a clean formal framework to analyze generative VLAs---the constraints of non-convexity, precision, and compounding errors are foundational to robotics~\cite{LozanoPerez1979,HsuLatombe1999, ross2010efficient} and our contribution is to study these constraints in the context of modern VLAs. Concretely, we contribute:
\begin{itemize}[leftmargin=*,itemsep=0pt,topsep=0pt]
\item A \textbf{topological impossibility result} for hallucination-free multi-mode coverage under continuous latent heads, and a quantitative \textbf{isoperimetric lower bound} connecting hallucination to decoder smoothness and mode separation.
\item A \textbf{precision-barrier analysis} for contact tasks, including a lower bound and a \textbf{generative precision trilemma} (fold, collapse, or hallucinate), plus a refinement-step tradeoff that clarifies why iterative diffusion/flow-style generation helps.
\item A reliability-aware analysis of \textbf{verification-guided planning} for \textbf{long-horizon tasks} that identifies when test-time compute helps and why adaptive search is needed under verifier noise.
\end{itemize}
Our theoretical analysis complements the growing body of empirical work on VLAs (e.g.,~\cite{black2024pi_0,molmoact,dai2025,jia2024towards}) and  helps clarify why VLAs work, as well as when and why they may fail. 
Our focus is {not} to develop new general mathematical theory, but rather, adapt existing tools to analyze the connection between generative models and embodied action generation. 
We discuss connections to recent experimental results and to practical system design choices, e.g., hybrid discrete-continuous structure, iterative refinement, and verification-guided planning.

\paragraph{Related Work.}
Hallucination has been widely studied in LLMs and VLMs, where it is typically framed as fluent generation that violates an external correctness oracle such as factuality, groundedness, or visual consistency~\cite{lei25hall,liu2025reducing,Xu2025,Kalai2024,Ji2022hallu}. We adopt this oracle-based view, but replace textual truthfulness with \emph{physical and task validity} in embodied settings.

Our analysis connects three previously separate threads. First, robotics has long recognized that feasible behavior is rarely convex or safely interpolatable: obstacles induce nonconvex and disconnected free spaces~\cite{LozanoPerez1979}, narrow passages make feasible connections hard to find~\cite{HsuLatombe1999,HsuLatombe2006}, and contact-rich manipulation often requires hitting thin, lower-dimensional contact sets and mode transitions~\cite{Mason1981,Mason1985,LozanoPerezMasonTaylor1984,hauser2010multi}. 
Second, prior work on deep generative models shows that continuous generators with connected latent priors cannot exactly represent disconnected target supports, leading to mode dropping or bridging samples~\cite{salmona2022can,Khayatkhoei2018,Tanielian2020,Issenhuth2023,Papamakarios2021,Stimper2022,aithal2024understanding}. In this work, we interpret these bridge samples physically; in robot action spaces, they can correspond to collisions, infeasible contacts, joint-limit violations, or actions that fall into non-progress regions. Our precision analysis likewise quantifies the density concentration needed to hit thin feasible tubes, and shows how flow/DDIM-style samplers realize this concentration through local contraction across refinement steps. 

Finally, our work relates to long-horizon decision-making and verification-guided search. Multi-step planning is computationally hard~\cite{Reif1979,Canny1988}, and compounding error is a classic source of long-horizon failure in imitation learning~\citep{ross2010efficient}. Viewed probabilistically, long-horizon success becomes a rare event, motivating decomposition and adaptive sampling methods~\citep{kahn1951estimation,rubinstein2004cross}. To see if VLAs can bypass these constraints through scale or test-time compute~\cite{kwok2025robomonkey,yoon2025monte,zhaollms23}, we formalize how horizon-induced rarity, verifier noise, and adaptive proposal mechanisms interact to reduce---or limit the reduction of---behavioral hallucinations.

\section{Preliminaries}

In this section, we formalize robot environments, task instances, and generative robot policies, then use this framework to define action hallucinations in embodied AI. Due to space constraints, we focus on the core ideas; full proofs, formal arguments, and additional details are deferred to the appendix.

\paragraph{Environment Model and Tasks.} We consider a robot that operates in an environment and seeks to achieve a goal.

\begin{definition}[Environment]
An environment is a tuple $\Eenv=(\Sspace,\A,\Omega,O,\mathcal T,\Csafe)$, 
where \(\Sspace\subseteq\R^n\) is the true physical state space,
\(\A\subseteq\R^d\) is the continuous action space, $\Omega$ is the observation space,
\(O:\Sspace\to \Omega \) is the observation map,
\(\mathcal T:\Sspace\times\A\to\Sspace\) is the transition map, and
\(\Csafe\subseteq\Sspace\) is the set of physically consistent and
constraint-satisfying states.
\end{definition}

\begin{definition}[Goal-reaching task instance]
A task instance is
$
    I=(\Eenv,s_0,\ell,\G,T),
$
where \(s_0\in\Csafe\) is the initial state, \(\ell\) is the language
instruction, \(\G\subseteq\Csafe\) is the goal set, and \(T\in\mathbb N\) is
the horizon or action budget.
\end{definition}

An action is generic in that it may denote a target pose, joint velocity, torque command, motion primitive, or action chunk. When \(a\) is a chunk
or trajectory, \(\mathcal T(s,a)\) denotes the terminal state after execution.
We use deterministic observations and dynamics for clarity; stochastic
observations or transitions can be handled by replacing \(O\) and
\(\mathcal T\) with kernels and taking the probabilities below over this
additional randomness. The deterministic setting is the special case where the kernels are point masses.
The set \(\Csafe\) is a ground-truth physical object used to define
hallucination. 

\begin{wrapfigure}{r}{0.48\textwidth}
    \centering
    \vspace{-2pt}
    \includegraphics[width=0.45\textwidth]{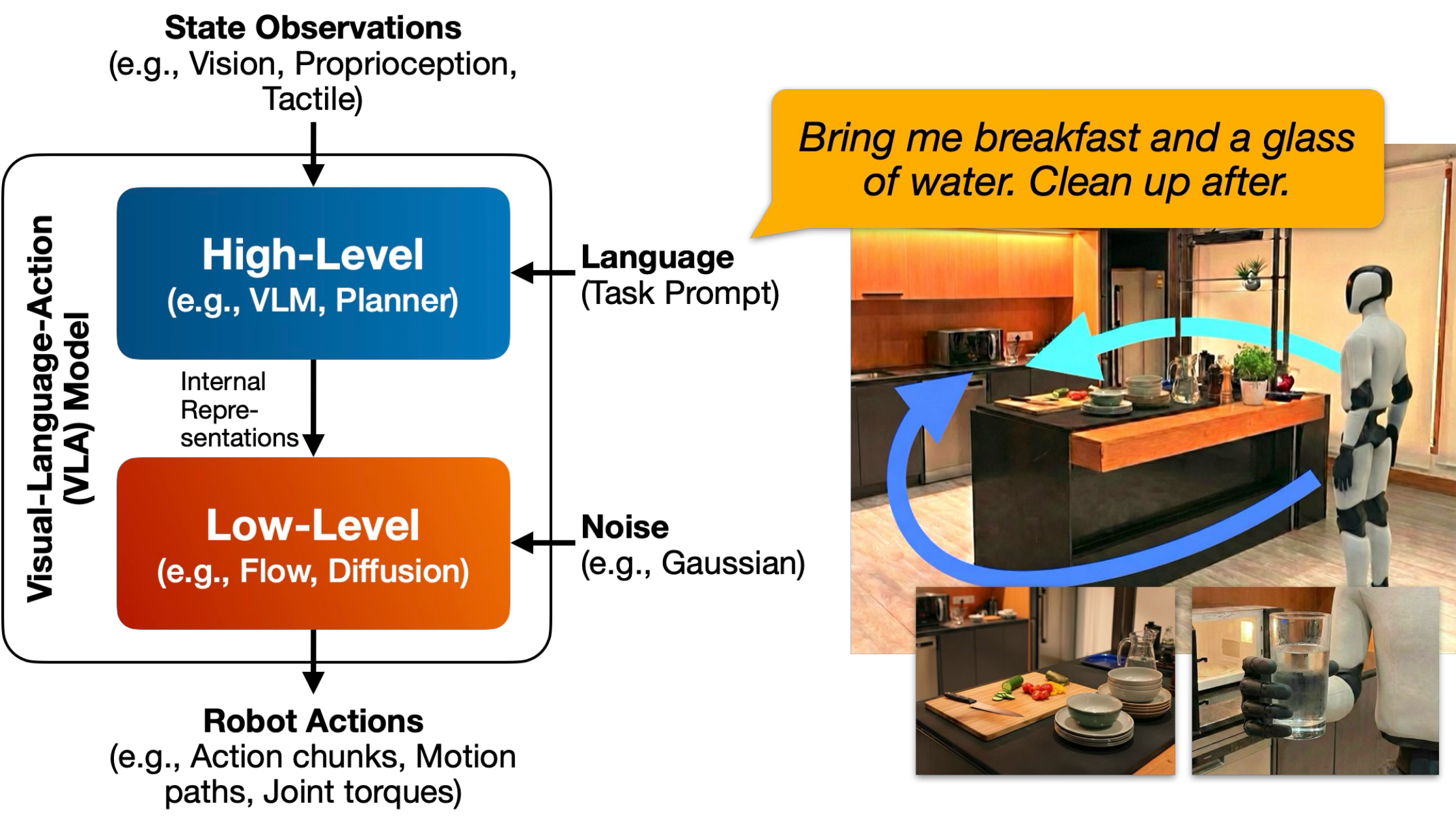}
    \caption{\small (Left) The prototypical generative VLA analyzed in this work. Given state observations, a task prompt, and a noise sample, the model outputs robot actions. Recent VLAs are structured into a high-level planner and a low-level action head, but part of our theory also applies to those that do not have this explicit structure (e.g., RDT~\cite{liu2025rdtb}). (Right) An example where a robot is given a long-horizon task that involves multi-modality and precision.}
    \label{fig:vla}
    \vspace{-10pt}
\end{wrapfigure}
\paragraph{VLA Contexts and Latent Action Heads.}
At time \(t\), the robot observes \(o_t=O(s_t)\). Let
$
    h_t=(o_0,a_0,\ldots,o_{t-1},a_{t-1},o_t)
$
be the observation-action history. A VLA encoder maps history and language to a
context
$
    c_t=E_\phi(h_t,\ell)\in\mathcal C .
$
The context may contain image-language embeddings, proprioception, memory, or
other internal representations. We define the oracle-state case as a special case
\(c_t=c^\star(s_t,\ell)\). %
This setting deliberately \emph{favors} the policy: the action head receives all
task-relevant state information, allowing us to isolate limitations of action
generation itself rather than confounding them with perceptual errors or state
estimation failures. We state our main results for this oracle-state
specialization, but the same arguments extend to more general settings by
replacing \(s\) with the history/context.

\begin{assumption}[Latent prior]
\label{ass:latent-prior}
The latent space $\Zspace$ is a nonempty, open, and path-connected subset of $\R^m$.
The latent prior $p_Z$ is absolutely continuous with respect to $m$-dimensional Lebesgue measure, and has a density $\rho_Z$ such that $\rho_Z(z) > 0$ for almost every $z \in \Zspace$, and is bounded below on compact subsets of $\Z$.

\end{assumption}

\begin{definition}[Latent-head VLA policy]
A latent-head VLA policy is a map
$
    \pi_\theta:\mathcal C\times\Z\to\A .
$
For each fixed context \(c\), the map \(z\mapsto \pi_\theta(c,z)\) is
continuous. The induced action distribution is
$
    p_\theta(\cdot\mid c)
    :=
    \big(\pi_\theta(c,\cdot)\big)_{\#}p_Z .
$
\end{definition}

This captures continuous generative action heads such as flow
matching, DDIM, and conditional flow decoders. 
For stochastic diffusion samplers with fresh noise injected at multiple denoising steps, one can regard the full collection of injected noise variables as the latent input. 
In the oracle-state case, we write
$
    p_\theta(\cdot\mid s,\ell)
    :=
    p_\theta(\cdot\mid c^\star(s,\ell)),
$
and omit \(\ell\) when it is fixed. 

\paragraph{Action and Plan Hallucination.} Given the structure of existing VLAs, we consider two major types of behavior hallucinations. At the low-level, an \emph{action hallucination} is an action (or action chunk / trajectory) whose induced rollout violates physical laws or environmental constraints (e.g., teleporting the robot, penetrating a solid obstacle, or exceeding kinematic limits). At the high-level, even physically valid actions can still fail to achieve the task goal. Hence, a \emph{plan hallucination} is a generated plan (sequence of actions) that violates constraints and/or fails to reach the goal. Plan hallucinations subsume action hallucinations but also include \emph{goal failure}. 

\begin{definition}[Physical validity oracle]
For a fixed environment \(\Eenv\), define
$
    f_{\mathrm{phys}}:\Sspace\times\A\to\{0,1\},
$
where \(f_{\mathrm{phys}}(s,a)=1\) if executing \(a\) from state \(s\) respects
the relevant physical and environmental constraints, and
\(f_{\mathrm{phys}}(s,a)=0\) otherwise.
\end{definition}

For chunks or trajectories, \(f_{\mathrm{phys}}\) checks all intermediate
states, not only the terminal state. State-dependent action/admissibility constraints,
such as torque limits and
task-specific feasibility constraints, are folded into \(f_{\mathrm{phys}}\).
Define the safe action set, $A_{\mathrm{safe}}(s)=\{a\in\A:f_{\mathrm{phys}}(s,a)=1\}$ and the forbidden/unsafe action set $A_{\mathrm{forb}}(s)=\A\setminus A_{\mathrm{safe}}(s)$.

\begin{definition}[Action hallucination]
At physical state \(s\) and context \(c\), a sampled action
\(a=\pi_\theta(c,z)\) is an action hallucination if
$
    f_{\mathrm{phys}}(s,a)=0.
$
Its probability is
$
    H_\theta(s,c)
    :=
    \Pr_{Z\sim p_Z}
    \left[
        f_{\mathrm{phys}}\big(s,\pi_\theta(c,Z)\big)=0
    \right].
$
In the oracle-state case, we write
$
    H_\theta(s,\ell):=H_\theta(s,c^\star(s,\ell)),
$
or simply \(H_\theta(s)\) when \(\ell\) is fixed.
\end{definition}

A closed-loop VLA induces a finite action sequence when unrolled. Given
\(I=(\Eenv,s_0,\ell,\G,T)\) and i.i.d. latents \(z_{0:T-1}\), define
    $o_t=O(s_t)$,
    $c_t=E_\phi(h_t,\ell)$, $a_t=\pi_\theta(c_t,z_t)$, and 
    $s_{t+1}=\mathcal T(s_t,a_t)$.
The induced plan is
$
    \tau_{\theta,\phi}(I,z_{0:T-1})
    :=
    (a_0,\ldots,a_{T-1})\in\A^T.
$
Here ``plan'' means the grounded action sequence induced by closed-loop
execution (which need not be a symbolic subtask sequence).

\begin{definition}[Plan hallucination]
For an action sequence \(\tau=(a_0,\ldots,a_{T-1})\), let
\(s_{t+1}=\mathcal T(s_t,a_t)\) be its rollout from \(s_0\). Define
\(f_{\mathrm{plan}}(s_0,\tau;\G)=1\) iff
$
    f_{\mathrm{phys}}(s_t,a_t)=1
    \quad\text{for all }t=0,\ldots,T-1,$
and
$
    s_T\in\G .
$
Otherwise \(f_{\mathrm{plan}}(s_0,\tau;\G)=0\). A plan hallucination occurs
when \(f_{\mathrm{plan}}(s_0,\tau;\G)=0\) and its probability on instance \(I\) is
$
    H^{\mathrm{plan}}_{\theta,\phi}(I)
    :=
    \Pr
    \left[
        f_{\mathrm{plan}}
        \big(
            s_0,
            \tau_{\theta,\phi}(I,Z_{0:T-1});
            \G
        \big)
        =0
    \right].
$
\end{definition}
We call \(I\) solvable if some \(\tau\in\A^T\) satisfies
\(f_{\mathrm{plan}}(s_0,\tau;\G)=1\), and focus on solvable instances.

\begin{figure*}
    \centering
    \includegraphics[width=0.99\linewidth]{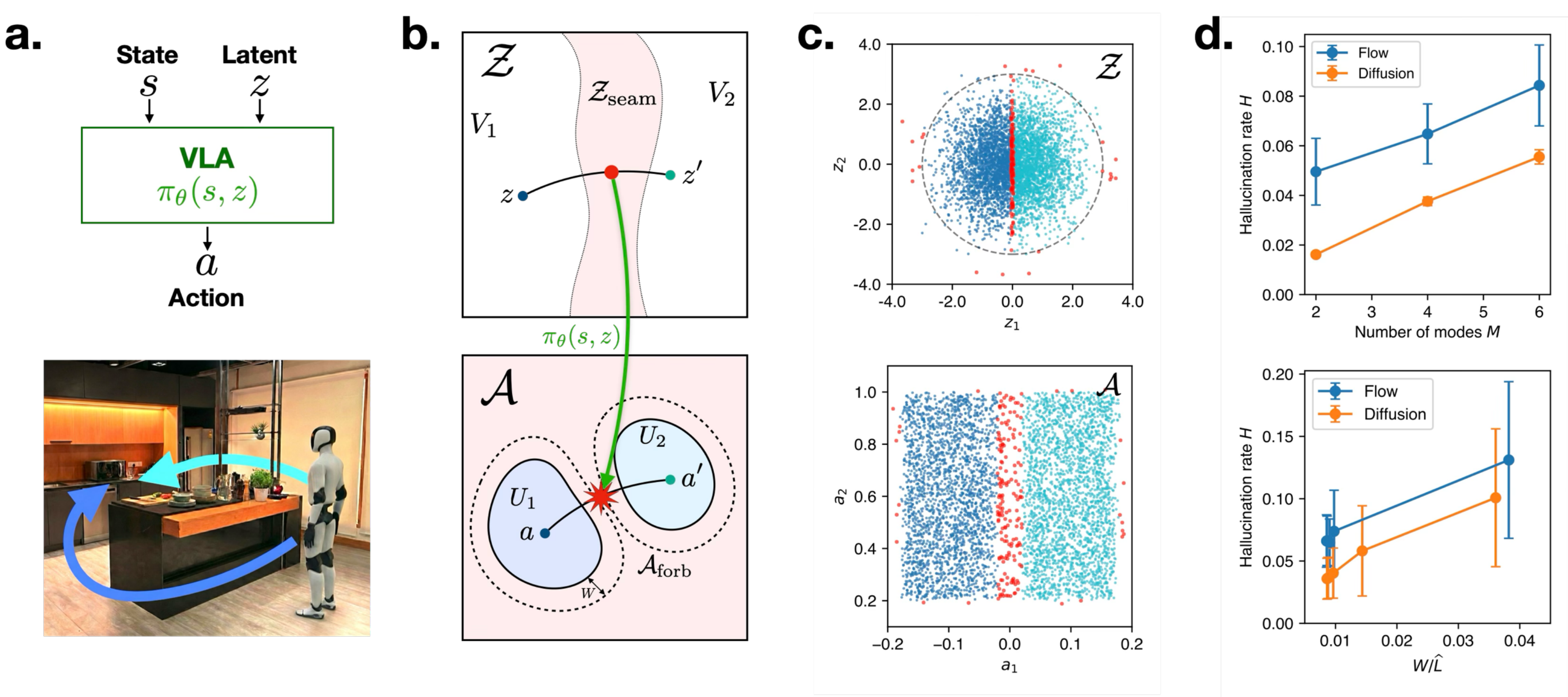}
    \caption{\small\textbf{Topological barrier for latent-variable VLA policies.}
(\textbf{a}) We study generative VLAs whose action head is a conditional latent-variable policy that maps a task-relevant state (or history of observations) and latent noise $z$ to a continuous action (or action chunk). In the illustrated navigation example, reaching the microwave requires going \emph{left} or \emph{right} around the kitchen island, inducing two qualitatively distinct modes of valid behaviors.
(\textbf{b}) Schematic of the topological barrier (Lemma~\ref{lem:topological}). The valid actions (bottom panel) decompose into disconnected components $U_1$ and $U_2$ (e.g., left vs.\ right), separated by a forbidden region $\mathcal A_{\mathrm{forb}}$. If $\pi_\theta(s,\cdot)$ is continuous and maps $z\mapsto a\in U_1$ and $z'\mapsto a'\in U_2$, then any continuous latent path between $z$ and $z'$ must cross an \emph{open seam} $\mathcal Z_{\mathrm{seam}}=\pi_\theta(s,\cdot)^{-1}(\mathcal A_{\mathrm{forb}})$, implying non-zero hallucination probability. %
(\textbf{c}) Diffusion model trained on bimodal action data: red points in $\mathcal Z$ (top) lie on the seam and decode to forbidden actions in $\mathcal A$ (bottom). See Appendix~\ref{app:band-topology-exp} for details. 
(\textbf{d}) Empirical trends for flow matching and diffusion. Hallucination rates $H$ can increase approximately linearly with the number of covered modes $M$ (top) and grows with $W/\widehat{L}$ (bottom), consistent with Theorem~\ref{thm:isoperimetry} ($\widehat{L}$ is the numerically-estimated Lipschitz constant).}
    \label{fig:topo}
    \vspace{-1em}
\end{figure*}

\section{Low-Level Action Hallucinations}
    \vspace{-0.4em}

\subsection{Topological Barrier for Multi-Mode Coverage}
\label{sec:topological}

Recent VLAs and generative policies are often lauded for their ability to represent complex \emph{multi-modal} action trajectories (e.g., ``go left'' or ``go right'' around an obstacle). However, experiments show these models still produce invalid actions despite significant training~\cite{zhai2025vfp,dai2025safe,jia2024towards} or   drop action modes~\cite{pan2025adonoisingdispellingmyths}. Figure \ref{fig:topo} provides a high-level overview of our theoretical results below that explains why this can occur.

Fix a physical state \(s\in\Sspace\) and instruction $\ell$. 
We first consider situations where the valid actions at a fixed state decompose into two distinct safe `modes', $U_L$ and $U_R$ (e.g., two inverse-kinematics branches, two grasp approach directions), separated by physically forbidden actions (e.g., joint-limit violation, infeasible contact). 

\begin{assumption}
\label{ass:disconnected}
For a given state $s$, assume that the safe action set $\A_{\mathrm{safe}}(s)$
  has at least two disjoint, nonempty path-connected components $U_L$ and $U_R$ in the subspace topology of $\A$. Also assume that the forbidden region $\A_{\mathrm{forb}}(s)$
  is open in $\A$ and has nonempty interior.
\end{assumption}

In this setting, topology dictates that a latent action head that attempts to cover both modes must hallucinate. 

\begin{restatable}[Topological Barrier]{lemma}{TopologicalBarrier}
\label{lem:topological}
Suppose Assumption \ref{ass:disconnected} holds and there exist $z_L,z_R \in \Zspace$ such that
$\pi_\theta(s,z_L) \in U_L$ and $\pi_\theta(s,z_R) \in U_R$.
Define the \emph{seam set}, 
$\Zspace_{\mathrm{seam}}(s) := \{ z \in \Zspace : \pi_\theta(s,z) \in \A_{\mathrm{forb}}(s)\}$.
Then $\Zspace_{\mathrm{seam}}(s)$ is nonempty and open in $\Zspace$. The hallucination probability at $s$ is strictly positive, 
  $
    H_\theta(s) = \Pr_{z \sim p_Z}\left[ z \in \Zspace_{\mathrm{seam}}(s)\right] > 0.
  $
\end{restatable}

Lemma~\ref{lem:topological} (proof in Appendix~\ref{app:proof:topo}) translates a geometric concept in deep generative modeling (e.g., ~\cite{Khayatkhoei2018,Tanielian2020,salmona2022can}) to robotics---if there are distinct modes of valid behaviors at a given state, then any generative policy that is smooth in its latent input must create a ``seam'' of \emph{in-between} latents that decode to \emph{in-between} actions. 
The policy cannot assign probability mass to both modes without also assigning probability mass to the invalid or unsafe actions between them. %

Next, we derive a quantitative lower bound in a multi-mode setting ($\geq 2$). Assumption \ref{ass:buffered-modes} below extends the previous setup to one where the valid action set $\A_\mathrm{safe}(s)$ splits into $M$ disconnected {modes} $U_1, U_2, \dots, U_M$. These modes in action space are separated by at least a distance $2W$ with a forbidden buffer of thickness $W$ around each mode.

\begin{assumption}
\label{ass:buffered-modes}
For a given state $s$, we assume that the safe set decomposes into
        $M \ge 2$ nonempty, pairwise-disjoint, closed, path-connected components
          $\A_{\mathrm{safe}}(s) = \bigsqcup_{i=1}^M U_i.$ Assume there exists $W>0$ such that
          $\operatorname{dist}(U_i,U_j) := \inf_{a\in U_i,a'\in U_j} \|a-a'\| \ge 2W \,\, \text{for all } i\neq j,$ 
        and for each $i$, the $W$-neighborhood of $U_i$ outside $U_i$ itself
        is forbidden:
        $
          \bigl\{ a \in \A :
            0 < \operatorname{dist}(a,U_i) < W
          \bigr\}
          \subseteq \A_{\mathrm{forb}}(s).
        $
\end{assumption}

In the common Gaussian latent setting, we can lower-bound the hallucination rate of smooth policies that are {locally} Lipschitz in $z$. %

\begin{restatable}[Isoperimetric lower bound on action hallucination]{theorem}{isoperimetric}
\label{thm:isoperimetry}
Fix a state $s$ satisfying Assumption~\ref{ass:buffered-modes}. Let $Z\sim\mathcal{N}(0,I_m)$ and fix a radius $R>0$.
Assume that the latent-to-action head $z\mapsto \pi_\theta(s,z)$ is $L$-Lipschitz on the typical-latent ball
$B_R := \{z\in\R^m : \|z\|\le R\}$, i.e.,
$  \|\pi_\theta(s,z)-\pi_\theta(s,z')\|
  \le L\,\|z-z'\|
  \,\, \forall z,z'\in B_R$ and $L>0$. 
Let $\epsilon := W/L$.
For each safe mode $U_i$, define the latent preimage $V_i := \pi_\theta(s,\cdot)^{-1}(U_i)$ and the in-ball mass
$p_i^{(R)} := \Pr\big[Z\in V_i\cap B_R\big].$ Let  $q(R)=\Pr\big[\|Z\|>R\big] $.
Then the hallucination probability at $s$ satisfies
\begin{equation}
  H_\theta(s)
  \ge
  \sum_{i=1}^M\left[\Phi\left(\Phi^{-1}(p_i^{(R)}) + \epsilon \right) - p_i^{(R)}\right] - q(R),
  \label{eq:iso-bound}
\end{equation}
where $\Phi$ is the CDF of the standard normal distribution. 
    
\end{restatable}

Compared to Lemma~\ref{lem:topological}, Theorem~\ref{thm:isoperimetry} provides a quantitative lower bound that reveals key factors that influence the action hallucination rate. As before, the problem is interpolation between the $M$ modes. The more smoothly the policy varies with the latent, the more probability mass leaks into the forbidden seam. In addition to the number of modes $M$, a key ratio to note is $\epsilon = W/L$. $W$ measures how wide the unsafe gap is in action space, and $L$ measures how quickly actions can change as we adjust the latent $z$. If the model is very smooth in the latent space, then it must transition gradually, and a larger set of latents produce intermediate invalid actions. 
The main idea underlying the proof (Appendix~\ref{app:proof:isoperimetric}) is to apply Gaussian isoperimetry to lower bound the mass of disjoint latent ``halos'' of radius $\epsilon=W/L$ around each safe-mode set.%

\mypara{Implications for VLAs} 
Our analysis reveals a tension between \emph{diversity} and \emph{safety}. Datasets can be multi-modal (e.g., distinct grasp types or approach trajectories) and a low-level continuous policy trained to cover all $M$ modes equally faces a geometric dilemma: (i) \textbf{Accept Hallucination:} To cover all diverse modes, the policy is geometrically required to expose significant Gaussian surface area to the forbidden regions, resulting in invalid ``interpolated'' actions; or (ii) 
    \textbf{Mode Collapse:} To reduce the hallucination rate (minimize $H_\theta$), the policy can assign very little mass to some modes or drop modes entirely (reduce $M$). %

Crucially, this dilemma is a geometric consequence and does \emph{not} disappear with better optimization or more data. %
One theoretical solution is to break the continuity assumption: generative action heads can employ \emph{hybrid} designs that use discrete decisions to switch between modes (e.g., discrete  tokens~\cite{saycan2022arxiv,pi0.6,hoeg2026hybrid}, diffusion options~\cite{feng25}, or mixture-of-experts (MoE) gating~\cite{hao2026abstracting}). 
Recent work provides suggestive empirical support for this prescription; VFP~\cite{zhai2025vfp} uses a MoE flow decoder to encourage mode-specific action generation, while Hybrid Diffusion Planner~\cite{hoeg2026hybrid} jointly samples discrete plan tokens and continuous trajectories to reduce long-horizon mode confusion. Both lead to better performance. 
In VLAs, the VLM/high-level planner can sample the mode before invoking the generative action head. However, this shifts the burden of hallucination avoidance to the planner (see Sec. \ref{sec:planning}). %
Theorem~\ref{thm:isoperimetry} also indicates that smoothness is a ``double-edged sword'' for contemporary models: on one hand, smoothness is generally encouraged in (learnt) policies as it is associated with more stable training~\cite{salmona2022can} and robot behavior---we usually prefer the case that similar $z$ imply similar actions. On the other hand, smoothness likely has a \emph{negative} effect on action hallucination; all else being equal, in~\eqref{eq:iso-bound}, each term $\Phi\bigl(\Phi^{-1}(p_i^{(R)})+\epsilon\bigr)-p_i^{(R)}$ is  increasing in $\epsilon=W/L$. %

\begin{figure*}
    \centering
    \includegraphics[width=0.99\linewidth]{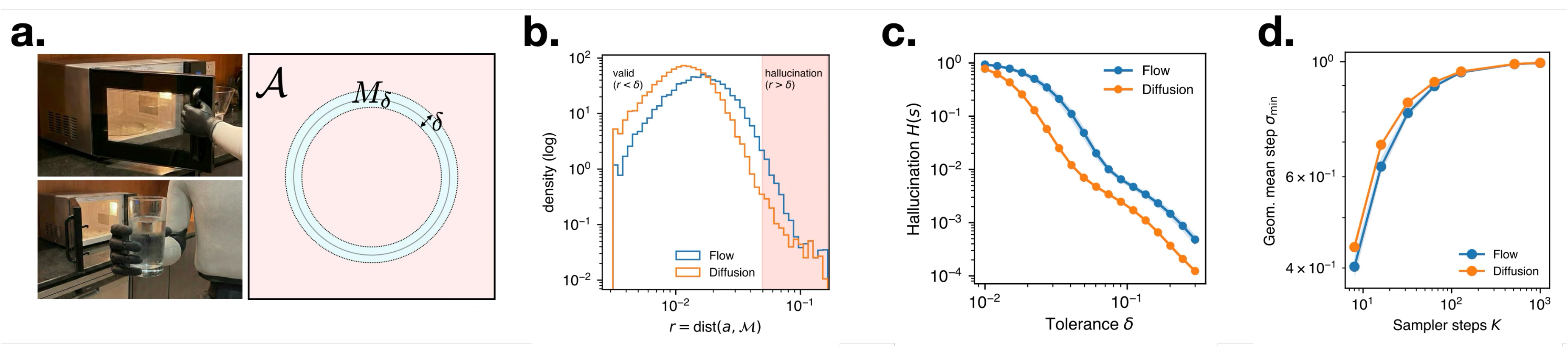}
    \caption{\small\textbf{Precision barrier for contact-rich tasks.}
    (\textbf{a}) Many manipulation tasks (e.g., grasping, peg-in-hole, handling tools / articulated / deformable objects) require high precision in that valid actions concentrate near a lower-dimensional feasible set. We model this as a $k$-dimensional manifold $\mathcal M\subset\mathcal A$ with tolerance tube $\mathcal M_\delta=\{a:\mathrm{dist}(a,\mathcal M)\le \delta\}$ (schematic).
    (\textbf{b}) Empirical distribution of distances $r=\mathrm{dist}(a,\mathcal M)$ for samples from flow matching and diffusion (log scales). The shaded region $r>\delta$ corresponds to action hallucinations. See Appendix \ref{app:prec_exps} for experiment details. 
    (\textbf{c}) Action hallucination rate $H(s;\delta)=\Pr[r>\delta]$ versus tolerance $\delta$ (log--log). Tightening tolerance sharply increases hallucinations, consistent with our precision barrier (Lemma~\ref{lem:density-barrier}) that shows maintaining low hallucination at small $\delta$ requires increasingly concentrated mass near $\mathcal M$.
    (\textbf{d}) The geometric mean of per-step minimum singular values
    increases toward $1$ as the number of sampler steps $K$ grows, indicating that the necessary overall contraction can be distributed across many mild refinement steps rather than a single severe collapse (Theorem~\ref{thm:precision-trilemma} and Corollary~\ref{cor:k-step-tradeoff}).}

    \label{fig:precision}
    \vspace{-1em}
\end{figure*}

\subsection{Precision Barrier for Contact Tasks}
\label{sec:codimension}

Let us now shift our attention from multiple modes to the complementary regime of precision and contact-rich tasks. %
Once contact occurs, rigid-body non-penetration, admissible contact modes, and practical limits such as maximum contact wrench induce additional state-dependent constraints on which actions are physically allowed. The set of admissible actions forms a thin tube around a lower-dimensional set in action space.
Here, we revisit this \emph{precision-barrier} from classical robotics~\cite{HsuLatombe2006} and we show that, as tolerance around the tube shrinks, any policy with a \emph{smooth, full-dimensional} action density is almost guaranteed to hallucinate actions (See Fig. \ref{fig:precision} for an overview). %

We model ideal contact configurations as a state-dependent $k$-dimensional submanifold $\M(s)$:

\begin{definition}
Let $\M(s) \subset \A$ be a compact $C^1$ submanifold of dimension $k < d$.
For $\delta > 0$, define the $\delta$-tube around $\M(s)$ by
 $\M_\delta(s) := \{ a \in \A : \operatorname{dist}(a,\M(s)) \le \delta \}$,
where $\operatorname{dist}(a,\M) = \inf_{x\in \M} \|a-x\|$.
For a state $s$, suppose the safe terminal actions satisfy
$\A_{\mathrm{safe}}(s;\delta) \subseteq \M_\delta(s)$ 
for a task-specific tolerance $\delta$. Define the $\delta$-tolerant success $S_\theta(s;\delta) := \Pr\left[a \in \A_{\mathrm{safe}}(s;\delta)\right]$, and hallucination probabilities, 
$H_\theta(s;\delta) := \Pr\left[a \notin \A_{\mathrm{safe}}(s;\delta)\right] = 1-S_\theta(s;\delta)$.
\label{def:precision}
\end{definition}

We fix a state $s$ and write $\M := \M(s)$ and $\M_\delta := \M_\delta(s)$. Assume that the dimensionality of the latent code $z$ is the same as the action, $m=d$ (as is the case in flow-matching and DDIM models). 
Below, we show that hitting a low-dimensional contact manifold requires the action distribution to concentrate its probability density.

\begin{restatable}[Precision Barrier]{lemma}{precisionbarrier}
\label{lem:density-barrier}
Fix a state $s$ and adopt Definition~\ref{def:precision}. 
Assume that the induced action law 
\(\mu_\theta(\cdot\mid s)\) is absolutely continuous with respect to
\(d\)-dimensional Lebesgue measure on \(\mathcal A\), with density
\(p(\cdot\mid s)\).
Then, there exist constants $C_\M>0$ and $\bar\delta>0$ (depending only on $\M$) such that for all $0<\delta\le \bar\delta$,
$S_\theta(s;\delta) \le C_\M \delta^{d-k}\cdot \esssup_{a\in \M_\delta} p(a\mid s)$,
and hence
\begin{align}
  H_\theta(s;\delta)
  \ge
  1 - C_\M\,\delta^{d-k}\cdot \esssup_{a\in \M_\delta} p(a\mid s).
\end{align}
\end{restatable}

\noindent
Lemma~\ref{lem:density-barrier} (proof in Appendix~\ref{app:precisionbarrier}) captures a source of action hallucination that is \emph{distinct} from the topological barrier in Section~\ref{sec:topological}. Even when the valid action set is connected (a single contact manifold), the policy needs to generate probability densities scaling as $O(\delta^{-(d-k)})$ inside the $\delta$-tube around~$\M$. This becomes more challenging as the tolerance $\delta$ shrinks or the codimension $(d-k)$ increases. 
The following theorem shows how a policy can supply this extreme density. 

\begin{restatable}[The Precision Trilemma: Collapse, Fold, or Hallucinate]{theorem}{trilemma}
\label{thm:precision-trilemma}
Let $m=d$ and let the action be generated by $a=F(z)$ where $Z \sim p_Z$.
Assume $F$ is $C^1$ on an open set $U \subseteq \Zspace$ containing the support of $p_Z$, with bounded latent density $\rho_Z(z) \le \rho_Z^{\max}$.
Define the active latent set as $Z_\delta := F^{-1}(\M_\delta) \cap U$.
Define local folding and local conditioning restricted to this active set: $N_\delta := \operatorname*{ess\,sup}_{a \in \M_\delta}\;
        \# \{z \in Z_\delta : F(z) = a\}$, and $\sigma_*(\delta) := \operatorname*{ess\,inf}_{z \in Z_\delta} \sigma_{\min}(J_F(z))$, 
where $J_F(z)$ is the Jacobian of $F$ and $\sigma_{\min}(J_F(z))$ its smallest singular value.
Assume $F$ is finite-to-one on $Z_\delta$ (i.e., $N_\delta < \infty$) and $\sigma_*(\delta)>0$.
Then the induced action distribution admits a density on $\M_\delta$ and the hallucination rate satisfies
$
  H_\theta(s;\delta)
  \;\ge\;
  1
  -
  C_\M\,\delta^{d-k}\;
  N_\delta\,\rho_Z^{\max}\,\sigma_*(\delta)^{-d}.
$
If the policy achieves a target hallucination level
$H_\theta(s;\delta)\le \eta$ for some $\eta\in(0,1)$, then the generator must satisfy
\begin{equation}
  \label{eq:tradeoff}
  \underbrace{N_\delta}_{\text{Fold}}
  \cdot
  \underbrace{\sigma_*(\delta)^{-d}}_{\text{Collapse}}
  \;\ge\;
  \frac{1-\eta}{C_\M \rho_Z^{\max}}
  \, \delta^{-(d-k)}.
\end{equation}
Hence, as precision tightens ($\delta \to 0$), any policy that maintains $H_\theta(s;\delta)\le \eta$ must either \emph{fold} space ($N_\delta \to \infty$) or \emph{collapse} volume locally ($\sigma_*(\delta) \to 0$).
\end{restatable}

Theorem~\ref{thm:precision-trilemma} (proof in Appendix~\ref{app:trilemma}) shows that near the tube, the induced action density satisfies $p(a\mid s)\ \lesssim\ N_\delta \cdot \rho_Z^{\max}\cdot \sigma_*(\delta)^{-d}$ so the two architectural mechanisms to increase density are:
(i) \textbf{Collapse} {(Jacobian contraction):} make $\sigma_*(\delta)$ small so $F$ compresses latent volume into a tiny region of action space; or 
(ii) \textbf{Fold} {(many-to-one decoding):} make $N_\delta$ large so that many disjoint latent regions decode to the same actions.
If neither occurs, the policy cannot place enough mass in the tube and will hallucinate actions. In short, as $\delta\to 0$, the generative policy must either collapse, fold, or hallucinate.

Modern VLAs often use flow-matching to generate actions. In theory, these are diffeomorphic generators with invertible steps so folding is not possible ($N_\delta=1$) and meeting the density requirement at tight tolerances {requires} {Jacobian collapse}.
These methods apply a sequence of ``refinement'' (or denoising) steps and %
Corollary~\ref{cor:k-step-tradeoff} below (proof in Appendix~\ref{app:subsec:kstep-tradeoff}) shows that to maintain a fixed hallucination rate at increasingly tight tolerances, a flow-like policy must either accumulate sufficient latent-space contraction across its $K$ refinement steps or allow individual steps to become increasingly ill-conditioned.

\begin{restatable}[$K$-step precision tradeoff]{corollary}{ksteptradeoff}
\label{cor:k-step-tradeoff}
Assume the decoder is a $K$-step composition
$F^{(K)}=\Phi_{K-1}\circ\cdots\circ\Phi_{0}$,
where each $\Phi_t$ is a $C^1$ diffeomorphism and there exist constants $\lambda_t>0$ such that
$\sigma_{\min}(J\Phi_t(z))\ge \lambda_t$ for all $z$.
Let 
$\sigma_*(\delta)\ \ge\ \inf_{z\in Z_\delta}\sigma_{\min}(JF^{(K)}(z))
 \ge \prod_{t=0}^{K-1}\lambda_t$.
Then, for all $0<\delta\le\bar\delta$,
$H_\theta^{(K)}(s;\delta) \ge\ 1 - C_\M\rho_Z^{\max}\left(\prod_{t=0}^{K-1}\lambda_t\right)^{-d}\delta^{d-k}.$ 
In particular, if $\lambda_t\ge \lambda\in(0,1)$ for all $t$, then 
a necessary condition to bound the hallucination rate $H_\theta^{(K)}(s;\delta) \leq \eta\in(0,1)$ is
$K\ \ge\ \frac{1}{d\log(1/\lambda)}\left[(d-k)\log\frac{1}{\delta} + \log\frac{1-\eta}{C_\M\rho_Z^{\max}}\right]$. 
\end{restatable}

\mypara{Implications for VLAs} In this regime of contact-rich tasks, a policy \emph{should} concentrate probability mass onto the safe action manifold rather than preserve spurious variance that gives rise to action hallucinations. However, there is a cost. Theorem~\ref{thm:precision-trilemma} shows that achieving this requires either {Jacobian collapse} or {folding}. 
Both make the latent space a poor coordinate system for optimization. %
A possible path around the precision barrier is to \emph{project} the VLA's proposed action back onto the task’s feasible contact manifold (e.g., via a low-level controller, or external sensing like tactile feedback) or \emph{parameterize} the manifold (e.g., an atlas of local charts~\cite{jaillet2012path} or kinematically-valid trajectories~\cite{nvidia2026alpamayo}) so decoded actions stay on‑manifold by construction. 

Corollary~\ref{cor:k-step-tradeoff} suggests the iterative nature of modern flow and diffusion policies is not solely about sampling from complex action distributions, but rather, about decomposing an effectively singular projection into a product of gentle, well-behaved steps. 
The dependence on $\delta$ enters only through a $\log(1/\delta)$ factor, so the required number of refinement steps grows only logarithmically with tolerance.
Interestingly, very recent large-scale empirical evidence~\cite{pan2025adonoisingdispellingmyths} on  benchmarks supports this view and suggests that the advantage of diffusion/flow-style policies on challenging contact-rich tasks is strongly associated with iterative computation and  stochasticity injection, rather than distribution learning. Our theory serves to explain these empirical findings. %

\section{Long-Horizon Plan Hallucinations}
\label{sec:planning}

This section builds on the previous analysis of one-step action generation and extends it to long-horizon goal achievement. We now move one layer up to the VLM/VLA planner and develop a theoretical framework that complements empirical findings showing that VLMs/VLAs suffer sharp performance degradation on long-horizon tasks (e.g., ~\cite{dai2025,Zhang2025ICCV,yang2025embodiedbench,kwok2025robomonkey}) and recent explorations into test-time compute (e.g., ~\cite{kwok2025robomonkey,yoon2025monte}). 

We begin by fixing a task instance $I$ with deterministic dynamics $\mathcal{T}$ and defining time-bounded backward-reachable sets, $\Sigma_0 := \G$,
$
\Sigma_t := \left\{ s\in\Sspace : \exists a\in\A \ \text{s.t.}\ f_{\mathrm{phys}}(s,a)=1 \,\, \text{and} \,\, \mathcal{T}(s,a)\in \Sigma_{t-1}\right\}$ for $t\ge 1.$
Thus, $s\in\Sigma_t$ means that from state $s$ there exists a physically-valid action that lands in a state from which the goal is reachable in $t$ further steps. %
For a state $s$ and remaining horizon $t\ge 1$ define the \emph{progress set},
$\A_{\mathrm{prog}}(s,t) :=
\left\{ a\in\A:\ f_{\mathrm{phys}}(s,a)=1 \ \text{and}\ \mathcal{T}(s,a)\in\Sigma_{t-1}\right\}$. 
By construction, $s\in\Sigma_t$ if and only if $\A_{\mathrm{prog}}(s,t)\neq\emptyset$, and  $\A_{\mathrm{prog}}(s,t)\subseteq \A_{\mathrm{safe}}(s)$, i.e., progress actions are a subset of safe actions. %

\begin{restatable}[Horizon Barrier]{lemma}{horizonbarrier}
\label{lem:horizon-barrier}
Let $(S_t)_{t=0}^{T}$ be the rollout from $s_0$. Define
$p_t(s)\ :=\ \Pr_{Z\sim p_Z}\big[\pi_\theta(s,Z)\in \A_{\mathrm{prog}}(s,t)\big]$ and $\gamma_t\ :=\ \sup_{s\in\Sigma_t} p_t(s)$.
Then the rollout success probability factorizes and is upper bounded by 
 $\Pr\left[f_{\mathrm{plan}}(s_0,\tau_\theta(s_0,Z_{0:T-1}))=1\right] 
  \le \prod_{t=1}^{T}\gamma_t.$ Hence, $H^{\mathrm{plan}}_{\pi_\theta}(I)  \ge\ 1-\prod_{t=1}^{T}\gamma_t$.
In particular, if $\gamma_t\le \gamma<1$ for all $t$, then
$\Pr[f_{\mathrm{plan}}(s_0,\tau_\theta)=1]\le \gamma^T$
and
$H^{\mathrm{plan}}_{\pi_\theta}(I)\ge 1-\gamma^T$.
\end{restatable}

Lemma~\ref{lem:horizon-barrier} (proof in Appendix~\ref{app:horizon-barrier-proof}) formalizes the horizon barrier as a \emph{compounding reachability bottleneck}. The idea is intuitive and well-known in robotics~\cite{ross2010efficient}: a length-$T$ valid rollout requires sampling a progress action at every step, so success is a product of {conditional} progress factors. Even if each $\gamma_t$ is only slightly below one, their product decays rapidly with horizon. This connects directly to the topological and precision barriers 
since both reduce per-step progress probability and the deficit compounds with $T$ (also see the discussion in Appendix~\ref{app:chunking} on action chunking).

Now suppose that a VLA can expend \emph{test-time compute} to \emph{propose} and \emph{verify}
candidates or continuations, as in sampling, beam/tree search, or using simulator/world-model
rollouts. Let $\mathcal P_I$ be
a measurable space of complete grounded candidate continuations for instance
$I$; each $\tau\in\mathcal P_I$ induces an action sequence, so
$f_{\mathrm{plan}}(s_0,\tau;\G)$ is well-defined. Define the valid set
$
        V:=\{\tau\in\mathcal P_I:\ f_{\mathrm{plan}}(s_0,\tau;\G)=1\}.
$
We study a \emph{verification-guided planner} that runs for at most $q$ rounds. At round $j$, based on
its search history, it chooses a proposal distribution $Q_j$ over $\mathcal P_I$,
samples $\tau_j\sim Q_j$, queries a randomized verifier $\tilde f_j(\tau_j)$, and
outputs the first accepted candidate; if no candidate is accepted in $q$ rounds,
it abstains and outputs $\bot$.
We assume verifier errors are uniformly bounded: given any history and proposed $\tau$, invalid candidates are accepted with probability at most $\varepsilon_{\mathrm{fp}}$,
and valid candidates are accepted with probability at least
$c:=1-\varepsilon_{\mathrm{fn}}$.

Define hallucination and abstention probabilities, 
$H(I):=\Pr[\hat\tau\neq\bot\ \wedge\ \hat\tau\notin V]$ and 
$A(I):=\Pr[\hat\tau=\bot]$.
The valid-output probability is $S(I)=1-A(I)-H(I)$. Plan 
hallucination here differs slightly from Lemma~\ref{lem:horizon-barrier} because the
planner may abstain. We focus on the design of $(\alpha,\beta)$-reliable VLAs: 

\begin{definition}[$(\alpha,\beta)$-reliability]
\label{def:alpha-beta-reliability}
Fix $\alpha,\beta\in[0,1]$ with $\alpha+\beta\le 1$. We say the planner is
\emph{$(\alpha,\beta)$-reliable} on instance $I$ if $H(I)\le\alpha$ and
$A(I)\le\beta$. Hence, an $(\alpha,\beta)$-reliable planner outputs a valid plan
with probability at least $S(I)\ge 1-\alpha-\beta$.
\end{definition}

Let $C_{j-1}$ be the event that the planner reaches round $j$. Define the
round-$j$ \emph{conditional valid mass} by
$
        \rho_j:=\Pr[\tau_j\in V\mid C_{j-1}]
        =\E[Q_j(V)\mid C_{j-1}],
$
 with the convention $\rho_j=0$ if $\Pr[C_{j-1}]=0$. We refer to
$(\rho_j)_{j=1}^q$ as the planner's \emph{amplification schedule}.

\begin{restatable}[Reliability search-budget window]{lemma}{realwindow}
\label{thm:amplification-search-reliability}
Consider any verification-guided planner with budget $q$, verifier error rates
$(\varepsilon_{\mathrm{fp}},\varepsilon_{\mathrm{fn}})$, and amplification schedule
$(\rho_j)_{j=1}^q$. Let $c:=1-\varepsilon_{\mathrm{fn}}$. Then
$
H(I)\le
\varepsilon_{\mathrm{fp}}
\sum_{j=1}^q \Pr[C_{j-1}](1-\rho_j)
\le q\varepsilon_{\mathrm{fp}}
$
and
$
A(I)\le \prod_{j=1}^q(1-c\rho_j).
$
Consequently, the planner is certified as $(\alpha,\beta)$-reliable whenever
$
\varepsilon_{\mathrm{fp}}
\sum_{j=1}^q \Pr[C_{j-1}](1-\rho_j)\le \alpha$ and 
$\prod_{j=1}^q(1-c\rho_j)\le\beta.
$
The simpler sufficient certificate
$q\varepsilon_{\mathrm{fp}}\le\alpha$ and
$\prod_{j=1}^q(1-c\rho_j)\le\beta$ also certifies
$(\alpha,\beta)$-reliability.
\end{restatable}

Lemma~\ref{thm:amplification-search-reliability} separates verifier-limited and
planner-limited effects. Each query creates a chance to accept a valid plan, but also increases the worst-case (upper-bounded) cumulative risk of a hallucination. Thus test-time compute is useful only if the proposal mechanism increases $\rho_j$ fast enough before the certified hallucination budget is exhausted. A natural follow-up question is how fast a planner needs to grow $\rho_j$? The following theorem provides an answer for long horizon tasks.

\begin{restatable}[Amplification-rate to beat the horizon barrier]{theorem}{amprate}
\label{thm:amplification-threshold}
Consider the planner above, with amplification schedule $(\rho_j)_{j=1}^q$. 
Any
$(\alpha,\beta)$-reliable planner must satisfy 
$
        \sum_{j=1}^q \Pr[C_{j-1}]\rho_j\ge 1-\alpha-\beta,
$
Assume an exponentially small initial valid mass
$\rho_1\le\gamma^T$ for some $\gamma\in(0,1)$. Then \textup{(i)} \textbf{Polynomial amplification is insufficient.}
If $\rho_j\le \rho_1j^p$ for some fixed $p\ge0$, then meeting the necessary
condition requires
$
q\ge
\left(1+\frac{(p+1)(1-\alpha-\beta)}{\gamma^T}\right)^{1/(p+1)}-1. 
$
For fixed $p$, $\gamma\in(0,1)$, and nontrivial reliability target
$1-\alpha-\beta>0$, this lower bound grows exponentially in $T$; \textup{(ii)} \textbf{Geometric amplification is sufficient, if within the verifier budget.}
Assume $c:=1-\varepsilon_{\mathrm{fn}}>0$, $\beta\in(0,1)$, and
$\rho_j\ge \min\{1,\rho_1r^{j-1}\}$ for some $r>1$. Then
$
q\varepsilon_{\mathrm{fp}}\le\alpha$ and 
$\sum_{j=1}^q \min\{1,\rho_1r^{j-1}\}
    \ge c^{-1}\log(1/\beta)
$
certify $(\alpha,\beta)$-reliability. In the unsaturated regime
$\rho_1r^{q-1}\le1$, the second condition becomes
$
        \rho_1\frac{r^q-1}{r-1}\ge c^{-1}\log(1/\beta)
$ 
and it suffices that
$
q\ge
\log_r\!\left(1+\frac{(r-1)c^{-1}\log(1/\beta)}{\rho_1}\right).
$
Thus, for fixed $r>1$, $c>0$, and $\beta\in(0,1)$, the geometric
cumulative-mass condition is achieved at $q=\Theta(\log 1/\rho_1)$ rounds. If $\rho_1\asymp\gamma^T$, this becomes  $q=\Theta(T)$.
\end{restatable}

Theorem~\ref{thm:amplification-threshold} highlights a certificate-level
feasibility constraint. When the horizon barrier makes the initial valid mass
small, amplification that grows only polynomially with the number of rounds (e.g., simple accept/reject) 
cannot accumulate enough valid mass without \emph{exponentially} many queries. At the
same time, Lemma~\ref{thm:amplification-search-reliability} gives the 
sufficient hallucination certificate $q\varepsilon_{\mathrm{fp}}\le\alpha$. If
$\varepsilon_{\mathrm{fp}}$ is bounded away from zero, the certificate window
can \emph{close} for large $T$. In contrast, geometric amplification grows valid mass
multiplicatively and can reach the abstention threshold quickly enough. %

\mypara{Implications for VLAs} 
Our analysis above provides us with key takeaways. First, long-horizon planning is difficult because success is a product of {conditional} progress factors. 
This problem has been long recognized in robotics and remains applicable to contemporary VLAs.
A practical (and historically natural) way to address the horizon barrier is to leverage hierarchy, e.g., verify and replan at intermediate semantic or skill boundaries, which turns one length-$T$ rare-event into many shorter-horizon problems. A complementary approach is to verify actions/plans. 
The verification-guided framework in Lemma~\ref{thm:amplification-search-reliability} clarifies what test-time compute can and cannot buy with noisy verifiers. 
The prescriptive message is that high-level VLM/VLA planners should not treat
verification as a binary filter on complete trajectories (more in Appendix~\ref{app:nonadaptive-window}). They should use \emph{feedback} 
to reshape $Q_j$ so that $\rho_j$ grows \textit{multiplicatively}. 
In robotics, gains often come from exploiting \emph{structure}; planners use verifier tests/scoring (e.g., partial constraints or intermediate progress tests) to prune large families of continuations at once (e.g., via tree search). See Appendix~\ref{app:family-level-pruning} for a short but more formal discussion.
These methodologies have been recently adopted in LLMs (e.g., ~\cite{yaotot23,zhou2024language,snell2024scalingllmtesttimecompute}) and are finding their way into VLAs~\cite{yoon2025monte,zhaollms23}, and our theory makes clear when and why search is important.

\section{Conclusions, Discussion, and Future Work}
\label{sec:discussion}

Taken together, our results suggest that robust VLAs are unlikely to emerge from scale alone. A core challenge is \emph{structural}: hallucination-free and successful behavior lives on disconnected and thin subsets of action space, and long horizons turn small local errors into global failures. Our analysis points to design principles more specific than ``scale the data'', ``increase the number of layers'', or ``add a verifier'': we should design VLA systems so that test-time computation is spent on {multiplicative amplification} of \emph{valid} behavior under controlled verification error, and \emph{structure} robot architectures to make that amplification possible.

\mypara{Limitations and future work} We view this work as a step toward a broader theory of VLAs and other RFMs such as World Action Models. Extending our theory to encompass perception, memory, and stochasticity is a natural step forward. Our verification-guided planning model also simplifies planning and verification; richer models could tighten amplification requirements and yield sharper algorithmic guidance. Beyond verification, future work can examine other reasoning approaches~\cite{zawalski2024robotic,molmoact,nvidia2026alpamayo}. 
Finally, it would be interesting to extend our theory towards task distributions and to training/test-time interventions that reduce hallucinations. Progress along these lines can guide the development of next-generation VLAs that are not only more capable, but more \emph{trustworthy}.

\bibliographystyle{unsrtnat}
\bibliography{references}

\clearpage 

\appendix

\section*{\Large Supplementary Material for ``Action Hallucination in Generative Vision-Language-Action Models''}

\section{Topological Barrier: Proofs and Additional Details}

\subsection{Proof of Topological Barrier (Lemma~\ref{lem:topological})}
\label{app:proof:topo}

\TopologicalBarrier*

\begin{proof}
We fix a state $s$ and suppress the dependence on $s$ to avoid clutter.

\paragraph{Nonemptiness.}
By Assumption~\ref{ass:latent-prior}, $\Zspace$ is path-connected, so there exists a continuous path
$\gamma:[0,1]\to \Zspace$ with $\gamma(0)=z_L$ and $\gamma(1)=z_R$.
Define the action-space path $g:[0,1]\to\A$ by $g(t):=\pi_\theta(s,\gamma(t))$.
By continuity of $z\mapsto \pi_\theta(s,z)$, the map $g$ is continuous.

By assumption, $g(0)\in U_L\subseteq \A_{\mathrm{safe}}(s)$ and $g(1)\in U_R\subseteq \A_{\mathrm{safe}}(s)$.
If $g(t)\in \A_{\mathrm{safe}}(s)$ for all $t\in[0,1]$, then $g$ would be a continuous path in $\A_{\mathrm{safe}}(s)$
connecting a point in $U_L$ to a point in $U_R$, contradicting that $U_L$ and $U_R$ are distinct path-connected components
of $\A_{\mathrm{safe}}(s)$ (Assumption~\ref{ass:disconnected}(i)).
Therefore, there exists $t^\star\in(0,1)$ such that $g(t^\star)\notin \A_{\mathrm{safe}}(s)$, i.e.,
$g(t^\star)\in \A_{\mathrm{forb}}(s)$.
Let $z^\star:=\gamma(t^\star)$; then $z^\star\in \Zspace_{\mathrm{seam}}(s)$, so $\Zspace_{\mathrm{seam}}(s)$ is nonempty.

\paragraph{Openness.}
By Assumption~\ref{ass:disconnected}(ii), $\A_{\mathrm{forb}}(s)$ is open in $\A$, and by construction
$\Zspace_{\mathrm{seam}}(s)=\pi_\theta(s,\cdot)^{-1}(\A_{\mathrm{forb}}(s))$.
Since preimages of open sets under continuous maps are open, $\Zspace_{\mathrm{seam}}(s)$ is open in $\Zspace$.

\paragraph{Positive probability.}
Pick $z^\star\in\Zspace_{\mathrm{seam}}(s)$. Since $\Zspace_{\mathrm{seam}}(s)$ is open in $\Zspace$
and $\Zspace$ is open in $\R^m$, there exists $r>0$ such that the closed ball
$K:=\overline{B(z^\star,r)}\subset \Zspace_{\mathrm{seam}}(s)$ and $K\subset \Zspace$.
By Assumption~\ref{ass:latent-prior}, there exists $\rho_{\min}(K)>0$ such that
$\rho_Z(z)\ge \rho_{\min}(K)$ for almost every $z\in K$.
Hence
\begin{align*}
H_\theta(s) = \Pr[Z\in \Zspace_{\mathrm{seam}}(s)] \ge \Pr[Z\in K]  = \int_K \rho_Z(z)dz \ge \rho_{\min}(K)\Vol(K) > 0.
\end{align*}
\end{proof}

\subsection{Proof of Isoperimetric Lower-Bound (Theorem~\ref{thm:isoperimetry})}
\label{app:proof:isoperimetric}

Our quantitative bound follows the same isoperimetric / latent pre-image boundary methodology used to prove precision limitations for Lipschitz pushforward generative models~\cite{salmona2022can,Tanielian2020,Issenhuth2023}, but we adapt this methodology to state-conditional action generation and interpret boundary mass as unsafe-action probability under a physically meaningful forbidden buffer assumption, with a local-Lipschitz typical-set correction.

\isoperimetric*
\begin{proof}
Fix \(s\) and $\ell$. Let $B_R := \{z\in\R^m:\|z\|\le R\}$, 
\(Z\sim\mathcal N(0,I_m)\), and let \(\gamma_m(\cdot)\) denote
standard Gaussian measure on \(\R^m\), so that
\(\Pr[Z\in E]=\gamma_m(E)\) for measurable \(E\). 

For each mode \(U_i\), define 
  $A_i := V_i\cap B_R$,
  and
  $p_i^{(R)} := \gamma_m(A_i)$.
Since \(U_i\) is closed and \(z\mapsto \pi_\theta(s,z)\) is continuous,
\(V_i\) is closed. Hence \(A_i\) is closed, and in particular Borel, so
\(p_i^{(R)}\) is well-defined and Gaussian isoperimetry applies to
\(A_i\).

\paragraph{Step 1: Separation of in-ball preimages.}
Fix \(i\neq j\). If either \(A_i\) or \(A_j\) is empty, the claimed
separation is trivial. Otherwise, take any \(z\in A_i\) and
\(z'\in A_j\). Then \(z,z'\in B_R\), so Lipschitzness gives
$
  \|\pi_\theta(s,z)-\pi_\theta(s,z')\|
  \le L\|z-z'\|.
$ 
Moreover, \(\pi_\theta(s,z)\in U_i\) and
\(\pi_\theta(s,z')\in U_j\). By Assumption~\ref{ass:buffered-modes},
$
  \operatorname{dist}(U_i,U_j)\ge 2W,
$ 
and therefore
\[
  2W
  \le
  \|\pi_\theta(s,z)-\pi_\theta(s,z')\|
  \le
  L\|z-z'\|.
\]
Thus \(\|z-z'\|\ge 2W/L=2\epsilon\). Taking the infimum over
\(z\in A_i\) and \(z'\in A_j\) yields 
$
  \operatorname{dist}(A_i,A_j)\ge 2\epsilon.
$

\paragraph{Step 2: Disjoint halos and inclusion in the seam inside \(B_R\).}
Fix any radius \(r\in(0,\epsilon)\). For each \(i\), define the open
\(r\)-halo around \(A_i\) by
$
  S_i^r := \{z\in\R^m:\ 0<\operatorname{dist}(z,A_i)<r\}.
$
If \(A_i=\emptyset\), we take \(S_i^r=\emptyset\).

\emph{Disjointness.}
Suppose, for contradiction, that \(z\in S_i^r\cap S_j^r\) for some
\(i\neq j\). Then there exist \(v_i\in A_i\) and \(v_j\in A_j\) such that $
  \|z-v_i\|<r$ and $\|z-v_j\|<r$.
Hence
\[
  \|v_i-v_j\|
  \le
  \|v_i-z\|+\|z-v_j\|
  <2r
  <2\epsilon,
\]
contradicting \(\operatorname{dist}(A_i,A_j)\ge 2\epsilon\). Therefore
the sets \(\{S_i^r\}_{i=1}^M\) are pairwise disjoint.

\emph{Halo points inside \(B_R\) map to forbidden actions.}
Take any \(z\in S_i^r\cap B_R\). By definition of \(S_i^r\), there exists
\(v\in A_i\) such that
$
  \|z-v\|<r<\epsilon .
$ 
Since \(z,v\in B_R\), Lipschitzness gives
\[
  \|\pi_\theta(s,z)-\pi_\theta(s,v)\|
  \le
  L\|z-v\|
  <
  L\epsilon
  =
  W.
\]
Because \(v\in A_i\subseteq V_i\), we have
\(\pi_\theta(s,v)\in U_i\). Therefore
$ \operatorname{dist}(\pi_\theta(s,z),U_i)<W.
$ 
Also, \(z\in S_i^r\) implies \(z\notin A_i\). Since \(z\in B_R\) and
\(A_i=V_i\cap B_R\), it follows that \(z\notin V_i\), and hence
$
  \pi_\theta(s,z)\notin U_i.
$ 
Thus
$
  \pi_\theta(s,z)
  \in
  \mathcal A\setminus U_i
$ and 
$
\operatorname{dist}(\pi_\theta(s,z),U_i)<W.
$
By the forbidden-buffer condition in Assumption~\ref{ass:buffered-modes},
$
  \pi_\theta(s,z)\in \mathcal A_{\mathrm{forb}}(s).
$
Therefore \(z\in\Zspace_{\mathrm{seam}}(s)\), and we conclude
\[
  \bigcup_{i=1}^M (S_i^r\cap B_R)
  \subseteq
  \Zspace_{\mathrm{seam}}(s).
\]

\paragraph{Step 3: Gaussian isoperimetry.}
For \(r\in(0,\epsilon)\), let
$
  A_i^r := \{z\in\R^m:\operatorname{dist}(z,A_i)<r\}
$
denote the open \(r\)-neighborhood of \(A_i\). Since \(A_i\) is closed,
$
  S_i^r = A_i^r\setminus A_i .
$
Hence
\[
  \Pr[Z\in S_i^r]
  =
  \gamma_m(A_i^r)-\gamma_m(A_i)
  =
  \gamma_m(A_i^r)-p_i^{(R)}.
\]
By the Gaussian isoperimetric inequality,
\[
  \gamma_m(A_i^r)
  \ge
  \Phi\!\left(\Phi^{-1}(p_i^{(R)})+r\right).
\]
Therefore
\[
  \Pr[Z\in S_i^r]
  \ge
  \Phi\!\left(\Phi^{-1}(p_i^{(R)})+r\right)-p_i^{(R)}.
\]

\paragraph{Step 4: Sum halos and apply the tail correction.}
Since the halos \(S_i^r\) are pairwise disjoint,
\[
  \Pr\!\left[Z\in\bigcup_{i=1}^M S_i^r\right]
  =
  \sum_{i=1}^M \Pr[Z\in S_i^r].
\]
Moreover,
\begin{align*}
  \Pr\!\left[Z\in\bigcup_{i=1}^M (S_i^r\cap B_R)\right]
  &=
  \Pr\!\left[Z\in\bigcup_{i=1}^M S_i^r\right]
  -
  \Pr\!\left[
    Z\in
    \left(\bigcup_{i=1}^M S_i^r\right)\cap B_R^c
  \right] \\
  &\ge
  \Pr\!\left[Z\in\bigcup_{i=1}^M S_i^r\right]
  -
  \Pr[Z\in B_R^c].
\end{align*}
Combining this with Step 2 gives
\begin{align*}
  H_\theta(s)
  &=
  \Pr[Z\in\Zspace_{\mathrm{seam}}(s)] \\
  &\ge
  \Pr\!\left[Z\in\bigcup_{i=1}^M (S_i^r\cap B_R)\right] \\
  &\ge
  \sum_{i=1}^M \Pr[Z\in S_i^r]
  -
  \Pr[\|Z\|>R] \\
  &\ge
  \sum_{i=1}^M
  \left[
    \Phi\!\left(\Phi^{-1}(p_i^{(R)})+r\right)
    -
    p_i^{(R)}
  \right]
  -
  q(R).
\end{align*}
This bound holds for every \(r\in(0,\epsilon)\). Letting
\(r\uparrow\epsilon\) and using continuity of \(\Phi\), we obtain
\[
  H_\theta(s)
  \ge
  \sum_{i=1}^M
  \left[
    \Phi\!\left(\Phi^{-1}(p_i^{(R)})+\epsilon\right)
    -
    p_i^{(R)}
  \right]
  -
  q(R).
\]
\end{proof}

\mypara{Remark: Small $\epsilon$ regime} In the situation where $\epsilon$ is small, we can obtain that 
\begin{equation}
  H_\theta(s)
  \ge
  \epsilon\sum_{i=1}^M \phi\bigl(\Phi^{-1}(p_i^{(R)})\bigr)
  -
  q(R)
  -
  O(\epsilon^2),
  \label{eq:iso-linear}
\end{equation}
where $\phi$ is the standard normal PDF. To see this, fix $i$ and write $x_i := \Phi^{-1}(p_i^{(R)})$, so that $\Phi(x_i)=p_i^{(R)}$.
Taylor's theorem with Lagrange remainder gives
\[
  \Phi(x_i+\epsilon)
  =
  \Phi(x_i) + \epsilon\,\phi(x_i) + \frac{\epsilon^2}{2}\,\Phi''(\xi_i)
\]
for some $\xi_i$ between $x_i$ and $x_i+\epsilon$.
Since $\Phi'(x)=\phi(x)$ and $\Phi''(x)=-x\phi(x)$, and
$\sup_{x\in\R}|x\phi(x)| = 1/\sqrt{2\pi e}$, we obtain the uniform bound
\[
  \Phi(x_i+\epsilon)-p_i^{(R)}
  =
  \epsilon\,\phi(x_i) + O(\epsilon^2)
  \quad
  \text{with } |O(\epsilon^2)|\le \frac{\epsilon^2}{2\sqrt{2\pi e}}.
\]
Summing over $i$ and subtracting $\Pr[\|Z\|>R]$ yields~\eqref{eq:iso-linear}.

\mypara{Remark: Optimizing over the typical-set radius $R$} We can tighten the bound in Theorem~\ref{thm:isoperimetry}; for any $R>0$ such that $\pi(s,\cdot)$ is Lipschitz on $B_R$ with constant $L_R$, define $\epsilon_R = W/L_R$. Then \eqref{eq:iso-bound} holds with $\epsilon$ replaced by $\epsilon_R$ and  $H_\theta(s)$ is lower bounded by the supremum of the RHS over $R$,
\begin{align*}
H_\theta(s) \geq \sup_{R > 0: L_R < \infty} \left\{ \sum_{i=1}^M \left[\Phi\left( \Phi^{-1}(p_i^{(R)})  + \epsilon_R \right)  - p_i^{(R)} \right] - q(R) \right\}
\end{align*}

\section{Precision Barrier: Proofs and Additional Details}
\label{app:precision-barrier-all}

\subsection{Proof of Precision Barrier (Lemma~\ref{lem:density-barrier})}
\label{app:precisionbarrier}

We will make use of the following lemma about the volume of tubular neighborhoods, which is based on Theorem A in~\cite{cannarsa2010minkowski}. 

\begin{lemma}[Volume of Tubes Around Submanifolds]
\label{lem:tube}
Let $\M \subset \R^d$ be a compact $C^1$ submanifold of dimension $k < d$.
Then there exist constants $C_\M > 0$ and $\bar{\delta} > 0$ such that for all $0 < \delta \le \bar{\delta}$,
\[
  \operatorname{Vol}(\M_\delta) \le C_\M \delta^{d-k},
\]
where $\operatorname{Vol}$ denotes the $d$-dimensional Lebesgue measure.
\end{lemma}

\begin{proof}
For $\delta>0$, $\M_\delta$ is the closed $\delta$-neighborhood
of $\M$:
\[
\M_\delta = \overline{B}(\M,\delta) := \{x\in \R^d : \operatorname{dist}(x,\M)\le \delta\}.
\]
Let
\[
\omega_m := \operatorname{Vol}_m\big(B_{\R^m}(0,1)\big)
\]
denote the $m$-dimensional Lebesgue volume of the unit Euclidean ball. Since $\M$ is a compact $C^1$ $k$-dimensional submanifold of $\R^d$, it is $k$-rectifiable~\cite{mattila2021rect}. Hence Theorem~A~\cite{cannarsa2010minkowski} applies with $n=d$ and $p=k$,
and yields the existence of the (finite) limit
\[
\lim_{\delta\to 0^+}\frac{\operatorname{Vol}(\M_\delta)}{\omega_{d-k}\,\delta^{d-k}}
=:\operatorname{Area}_k(\M),
\]
where $\operatorname{Area}_k(\M)$ denotes the intrinsic $k$-dimensional surface area of $\M$ (i.e., the $k$-dimensional Hausdorff measure). Let $L_{\mathcal M} := \operatorname{Area}_k(\mathcal M)$ and define
\[
f(\delta):=\frac{\Vol(\M_\delta)}{\omega_{d-k}\,\delta^{d-k}}.
\]
Then $\lim_{\delta\to 0^+} f(\delta)=L_\M$. Taking $\varepsilon=1$ in the
definition of convergence, there exists $\delta_0>0$ such that for all
$0<\delta<\delta_0$,
\[
|f(\delta)-L_\M|<1 \quad\Rightarrow\quad f(\delta) < L_\M+1.
\]
Set $\bar\delta:=\delta_0/2$. Then for all $0<\delta\le \bar\delta$, $f(\delta)\le L_\M+1$. Multiplying through gives
\[
\Vol(\M_\delta)
\le \underbrace{\omega_{d-k}(L_\M+1)}_{C_\M}\delta^{d-k}
\]
\end{proof}

Using Lemma~\ref{lem:tube}, we prove our precision barrier below.

\precisionbarrier*
\begin{proof}
Fix $0<\delta\le \bar\delta$ (where $\bar\delta$ will be chosen below).
Since $\A_{\mathrm{safe}}(s;\delta)\subseteq \M_\delta$ and $p(\cdot\mid s)\ge 0$,
\[
  S_\theta(s;\delta)
  = \int_{\A_{\mathrm{safe}}(s;\delta)} p(a\mid s)\,da
  \le \int_{\M_\delta} p(a\mid s)\,da.
\]
Since
$p(a\mid s)\le \esssup_{u\in \M_\delta} p(u\mid s)$ for almost every $a\in \M_\delta$,
\[
  \int_{\M_\delta} p(a\mid s)\,da
  \le \Vol(\M_\delta)\cdot \esssup_{a\in \M_\delta} p(a\mid s).
\]
Applying Lemma~\ref{lem:tube} to the compact $C^1$
$k$-submanifold $\M\subset\R^d$ provides constants $C_\M>0$ and
$\bar\delta>0$ such that for all $0<\delta\le \bar\delta$,
\[
  \Vol(\M_\delta)\le C_\M\,\delta^{d-k}.
\]
Combining the above yields
\[
  S_\theta(s;\delta)\le C_\M\,\delta^{d-k}\cdot \esssup_{a\in \M_\delta} p(a\mid s),
\]
and therefore
\[
  H_\theta(s;\delta)=1-S_\theta(s;\delta)
  \ge 1 - C_\M\,\delta^{d-k}\cdot \esssup_{a\in \M_\delta} p(a\mid s).
\]
Note that if $H_\theta(s;\delta)\le \eta$, then $S_\theta(s;\delta)\ge 1-\eta$, and rearranging the success bound gives
\[
  \esssup_{a\in \M_\delta} p(a\mid s)
  \ge \frac{1-\eta}{C_\M}\,\delta^{-(d-k)}.
\]
\end{proof}

\subsection{Proof of Generative Trilemma (Theorem~\ref{thm:precision-trilemma})}
\label{app:trilemma}

\trilemma*
\begin{proof}
Consider the active set $Z_\delta = F^{-1}(\M_\delta) \cap U$. Since $\sigma_{\min}(J_F(z)) \ge \sigma_*(\delta) > 0$ for almost every $z \in Z_\delta$, the Inverse Function Theorem implies that $F$ is a local diffeomorphism in a neighborhood of any point in $Z_\delta$.
Since $\operatorname{supp}(p_Z) \subseteq U$, the density of the pushforward distribution $a = F(z)$ is fully determined by preimages in $U$. Using the Area Formula (change-of-variables for many-to-one maps), for almost every $a \in \M_\delta$:
\[
  p(a\mid s) =  \sum_{z \in F^{-1}(a) \cap U} \rho_Z(z) \left| \det J_F(z) \right|^{-1}.
\]
If $a \in \M_\delta$, then any preimage $z \in F^{-1}(a) \cap U$ lies in $Z_\delta$ by definition.
We bound the terms uniformly:
\begin{enumerate}[itemsep=1pt,topsep=1pt]
\item {Density:} $\rho_Z(z) \le \rho_Z^{\max}$.
\item {Jacobian:} for $z \in Z_\delta$,
$
|\det J_F(z)| \ge (\sigma_{\min}(J_F(z)))^d \ge \sigma_*(\delta)^d.
$
\end{enumerate}
Substituting these bounds into the sum gives, for a.e.\ $a\in\M_\delta$,
\begin{align}
  p(a\mid s) & \le \sum_{z \in F^{-1}(a) \cap U}\frac{\rho_Z^{\max}}{\sigma_*(\delta)^d} \nonumber \\
  & = \#\{z \in Z_\delta : F(z)=a\}\cdot\rho_Z^{\max}\cdot\sigma_*(\delta)^{-d}. \nonumber
\end{align}
Taking the essential supremum over $a \in \M_\delta$ yields
\[
  \operatorname*{ess\,sup}_{a \in \M_\delta} p(a\mid s) \le N_\delta \rho_Z^{\max} \sigma_*(\delta)^{-d}.
\]
Applying Lemma~\ref{lem:density-barrier},
\[
  H_\theta(s;\delta) \ge 1 - C_\M\,\delta^{d-k}\cdot\operatorname*{ess\,sup}_{a \in \M_\delta} p(a\mid s).
\]
Finally, if $H_\theta(s;\delta)\le \eta$, then the lower bound 
must not exceed $\eta$, i.e.,
\[
  1 - C_\M\delta^{d-k}N_\delta\rho_Z^{\max}\sigma_*(\delta)^{-d}
  \le \eta,
\]
which we can rearrange to \eqref{eq:tradeoff}.
\end{proof}

\subsection{Proof of K-step tradeoff (Corollary~\ref{cor:k-step-tradeoff})}
\label{app:subsec:kstep-tradeoff}

\ksteptradeoff*

\begin{proof}
Since each $\Phi_t$ is a diffeomorphism, their composition $F^{(K)}$ is a diffeomorphism, hence $N_\delta=1$. The singular-value lower bound follows from the chain rule and the standard inequality $\sigma_{\min}(AB)\ge \sigma_{\min}(A)\sigma_{\min}(B)$. Apply Theorem~\ref{thm:precision-trilemma} with $N_\delta=1$ and $\sigma_*(\delta)\ge \prod_t\lambda_t$ to obtain the hallucination bound. Rearranging and taking logs gives the number of required steps $K$. Note that the bound is vacuous when the right-hand side is negative.
\end{proof}

\section{Planning Hallucinations: Proofs and Additional Details}

\subsection{Proof for Horizon Barrier (Lemma~\ref{lem:horizon-barrier})}
\label{app:horizon-barrier-proof}

\horizonbarrier*
\begin{proof}
Let $E_k:=\{\pi_\theta(S_k,Z_k)\in \A_{\mathrm{prog}}(S_k,T-k)\}$ for $k=0,\dots,T-1$. Under deterministic dynamics and the definition of $\A_{\mathrm{prog}}$, the plan is valid if and only if
$\bigcap_{k=0}^{T-1}E_k$ occurs. By the chain rule,
\[
\Pr\Big[\bigcap_{k=0}^{T-1}E_k\Big]
=
\prod_{k=0}^{T-1}\Pr[E_k\mid E_0\cap\cdots\cap E_{k-1}]
=
\prod_{t=1}^{T}\eta_t,
\]
where we re-indexed with $t=T-k$, i.e., $\eta_t := \Pr[E_{T-t}\mid E_0\cap\cdots\cap E_{T-t-1}]$

On the event $E_0\cap\cdots\cap E_{T-t-1}$, the rollout state $S_{T-t}$ must lie in $\Sigma_t$ by definition of
$\A_{\mathrm{prog}}(S_{T-t-1},t+1)$ and $\Sigma_t$. Therefore $p_t(S_{T-t})\le \sup_{s\in\Sigma_t}p_t(s)=\gamma_t$ pointwise on this event,
implying $\eta_t\le \gamma_t$. Hence
\[
\Pr[f_{\mathrm{plan}}=1]=\prod_{t=1}^{T}\eta_t\le \prod_{t=1}^{T}\gamma_t,
\]
and the hallucination bounds follow by complement. The $\gamma^T$ specialization is immediate when $\gamma_t\le\gamma$ for all $t$.
\end{proof}

\subsection{Proof for Reliability search budget window}
\label{app:proof:reliabilitysearchwindow}

\realwindow*
\begin{proof}
Let $\mathcal F_{j-1}$ denote the $\sigma$-field generated by the planner's
history up to round $j-1$, including past proposals and verifier outcomes. On
$C_{j-1}$, the proposal distribution $Q_j$ is $\mathcal F_{j-1}$-measurable and
$\tau_j\sim Q_j$.

\smallskip
\noindent\textbf{Hallucination bound.}
Let
\[
F_j:=\{C_{j-1}\ \wedge\ \tau_j\notin V\ \wedge\ \tilde f_j(\tau_j)=1\}
\]
be the event that round $j$ produces a false acceptance. If the planner outputs
an invalid candidate, then exactly one of the events $F_1,\ldots,F_q$ occurs.
Thus
\[
        H(I)=\sum_{j=1}^q \Pr[F_j].
\]
By the conditional false-positive guarantee, for every history and proposed
invalid candidate $\tau\in V^c$,
\[
\Pr[\tilde f_j(\tau)=1\mid \mathcal F_{j-1},\tau_j=\tau]
    \le \varepsilon_{\mathrm{fp}}.
\]
Therefore, conditioning on $\mathcal F_{j-1}$,
\begin{align*}
\Pr[F_j\mid\mathcal F_{j-1}]
&=\mathbf 1\{C_{j-1}\}
  \int_{V^c} Q_j(d\tau)\,
  \Pr[\tilde f_j(\tau)=1\mid\mathcal F_{j-1},\tau_j=\tau] \\
&\le
\mathbf 1\{C_{j-1}\}\varepsilon_{\mathrm{fp}}Q_j(V^c) \\
&=
\mathbf 1\{C_{j-1}\}\varepsilon_{\mathrm{fp}}(1-Q_j(V)).
\end{align*}
Taking expectations gives
\[
\Pr[F_j]
\le
\varepsilon_{\mathrm{fp}}
\E\big[\mathbf 1\{C_{j-1}\}(1-Q_j(V))\big].
\]
By the definition of $\rho_j$, with the stated convention when
$\Pr[C_{j-1}]=0$,
\[
\E\big[\mathbf 1\{C_{j-1}\}(1-Q_j(V))\big]
=
\Pr[C_{j-1}](1-\rho_j).
\]
Hence
\[
H(I)
\le
\varepsilon_{\mathrm{fp}}
\sum_{j=1}^q \Pr[C_{j-1}](1-\rho_j)
\le q\varepsilon_{\mathrm{fp}}.
\]

\smallskip
\noindent\textbf{Abstention bound.}
Let $C_j$ be the event that the planner reaches round $j+1$, equivalently that
no acceptance occurred in rounds $1,\ldots,j$. Then $A(I)=\Pr[C_q]$ and $C_0$ is
the sure event. Let $T_j$ be the event that round $j$ gives a true acceptance:
\[
T_j:=\{C_{j-1}\ \wedge\ \tau_j\in V\ \wedge\ \tilde f_j(\tau_j)=1\}.
\]
Using the conditional true-acceptance guarantee,
\begin{align*}
\Pr[T_j]
&\ge
c\,\E\big[\mathbf 1\{C_{j-1}\}Q_j(V)\big] \\
&=
 c\,\Pr[C_{j-1}]\rho_j.
\end{align*}
Since every true acceptance prevents reaching the next round,
\begin{align*}
\Pr[C_j]
&\le \Pr[C_{j-1}]-\Pr[T_j] \\
&\le \Pr[C_{j-1}](1-c\rho_j).
\end{align*}
This recursive inequality remains valid when $\Pr[C_{j-1}]=0$. Iterating from
$\Pr[C_0]=1$ yields
\[
        A(I)=\Pr[C_q]\le \prod_{j=1}^q(1-c\rho_j).
\]
The certification statements follow by requiring the corresponding upper bounds
on $H(I)$ and $A(I)$ to be at most $\alpha$ and $\beta$, respectively.
\end{proof}

\subsection{Non-adaptive proposal-and-verify window and horizon fragility}
\label{app:nonadaptive-window}

Lemma~\ref{thm:amplification-search-reliability} is a useful tool to study how
fast adaptation must take place. We first specialize it to the common
non-adaptive case, where the planner repeatedly samples from the same proposal
distribution and only uses the verifier as an accept/reject filter.

\begin{corollary}[Non-adaptive proposal-and-verify window and horizon fragility]
\label{cor:nonadaptive-window}
Consider the verification-guided planner of
Lemma~\ref{thm:amplification-search-reliability}. Assume it is non-adaptive in
the sense that, on every reached round, $Q_j=Q$ for a fixed proposal distribution
$Q$ over $\mathcal P_I$. Let $\rho:=Q(V)$ and
$c:=1-\varepsilon_{\mathrm{fn}}$. Then
\[
H(I)\le \varepsilon_{\mathrm{fp}}(1-\rho)
       \sum_{j=1}^q\Pr[C_{j-1}]
\le q\varepsilon_{\mathrm{fp}}(1-\rho)
\le q\varepsilon_{\mathrm{fp}},
\]
and
\[
        A(I)\le (1-c\rho)^q.
\]
Consequently, if $\beta\in(0,1)$, $c\rho\in(0,1)$, and
$\varepsilon_{\mathrm{fp}}(1-\rho)>0$, the sharper non-adaptive certification
window is
\[
        \frac{\log \beta}{\log(1-c\rho)}
        \le q
        \le
        \frac{\alpha}{\varepsilon_{\mathrm{fp}}(1-\rho)}.
\]
The cruder instance-independent upper bound $q\le\alpha/\varepsilon_{\mathrm{fp}}$
also suffices when $\varepsilon_{\mathrm{fp}}>0$. If
$\varepsilon_{\mathrm{fp}}(1-\rho)=0$, the non-adaptive false-positive certificate
imposes no upper bound on $q$.

Moreover, assume $c>0$, $\beta\in(0,1)$, and $\rho>0$. If $Q$ is induced by
rolling out a $T$-step sequential policy and the horizon barrier gives
$\rho=Q(V)\le\gamma^T$ for some $\gamma\in(0,1)$, then any $q$ satisfying the
abstention certificate $(1-c\rho)^q\le\beta$ obeys
\[
        q\ge \frac{\log(1/\beta)}{-\log(1-c\rho)}
        \ge \frac{\log(1/\beta)}{-\log(1-c\gamma^T)}.
\]
In particular, when $c\gamma^T\le 1/2$, this certificate requires
\[
        q\ge \frac{1}{2c}\gamma^{-T}\log(1/\beta),
\]
which is exponential in $T$.
\end{corollary}

\begin{proof}
In the non-adaptive case, $Q_j=Q$ on every reached round. For false acceptances,
the proof of Lemma~\ref{thm:amplification-search-reliability} gives directly
\[
\Pr[F_j]\le
\varepsilon_{\mathrm{fp}}\E[\mathbf 1\{C_{j-1}\}(1-Q(V))]
=
\varepsilon_{\mathrm{fp}}(1-\rho)\Pr[C_{j-1}],
\]
and summing over $j$ yields the stated hallucination bound.
Similarly, at each reached round the true-acceptance probability is at least
$c\rho$, so
\[
        \Pr[C_j]\le \Pr[C_{j-1}](1-c\rho).
\]
Iterating gives $A(I)\le(1-c\rho)^q$.

The certification window follows by enforcing
$q\varepsilon_{\mathrm{fp}}(1-\rho)\le\alpha$ and
$(1-c\rho)^q\le\beta$. Since $\beta\in(0,1)$ and $c\rho\in(0,1)$, taking logs in
the abstention inequality gives
$q\ge \log\beta/\log(1-c\rho)$.

For the horizon-fragility statement, the abstention certificate implies
\[
        q\ge \frac{\log(1/\beta)}{-\log(1-c\rho)}.
\]
The map $x\mapsto -\log(1-cx)$ is increasing on $[0,1/c)$ for $c>0$. Since
$\rho\le\gamma^T$, the denominator is at most $-\log(1-c\gamma^T)$, giving
\[
        q\ge \frac{\log(1/\beta)}{-\log(1-c\gamma^T)}.
\]
Finally, if $x:=c\gamma^T\le1/2$, then $-\log(1-x)\le2x$, and therefore
\[
        q\ge \frac{1}{2c}\gamma^{-T}\log(1/\beta).
\]
\end{proof}

\subsection{Cumulative valid-mass condition for abstention control}
\label{app:proof:cumulative-mass}

\begin{corollary}[Cumulative valid-mass condition for abstention control]
\label{cor:cumulative-mass}
Under the assumptions of Lemma~\ref{thm:amplification-search-reliability},
\[
        A(I)\le \exp\!\left(-c\sum_{j=1}^q\rho_j\right).
\]
If $c>0$ and $\beta\in(0,1)$, a sufficient condition for $A(I)\le\beta$ is
\[
        \sum_{j=1}^q\rho_j\ge c^{-1}\log(1/\beta).
\]
\end{corollary}

\begin{proof}
From Lemma~\ref{thm:amplification-search-reliability},
\[
A(I)\le \prod_{j=1}^q(1-c\rho_j).
\]
Since $c\rho_j\ge0$ and $1-x\le e^{-x}$ for all $x\ge0$,
\[
\prod_{j=1}^q(1-c\rho_j)
\le
\prod_{j=1}^q e^{-c\rho_j}
=
\exp\!\left(-c\sum_{j=1}^q\rho_j\right).
\]
The sufficient condition follows by rearranging
$\exp(-c\sum_j\rho_j)\le\beta$.
\end{proof}

\subsection{Proof for Amplification-rate to beat the horizon barrier}
\label{app:proof:amplification-threshold}

\amprate*
\begin{proof}
Define the valid set of plans, 
\[
E_{\mathrm{val}}:=\{\hat\tau\neq\bot\ \wedge\ \hat\tau\in V\},
\]
so $S(I)=\Pr[E_{\mathrm{val}}]$. If the planner outputs a valid candidate, then
some reached round must have proposed a valid candidate. Therefore
\[
E_{\mathrm{val}}
\subseteq
\bigcup_{j=1}^q\bigl(C_{j-1}\cap\{\tau_j\in V\}\bigr).
\]
By the union bound,
\begin{align*}
S(I)
&\le
\sum_{j=1}^q\Pr[C_{j-1}\cap\{\tau_j\in V\}] \\
&=
\sum_{j=1}^q\Pr[C_{j-1}]\rho_j
\le
\sum_{j=1}^q\rho_j.
\end{align*}
If the planner is $(\alpha,\beta)$-reliable, then
$S(I)=1-A(I)-H(I)\ge1-\alpha-\beta$. Hence
\[
        \sum_{j=1}^q\Pr[C_{j-1}]\rho_j\ge1-\alpha-\beta,
        \qquad
        \sum_{j=1}^q\rho_j\ge1-\alpha-\beta.
\]

\smallskip
\noindent\textbf{(i)}
If $\rho_j\le\rho_1j^p$ and $\rho_1\le\gamma^T$, then
\[
\sum_{j=1}^q\rho_j
\le
\gamma^T\sum_{j=1}^q j^p
\le
\gamma^T\frac{(q+1)^{p+1}-1}{p+1},
\]
where the last inequality follows from
$\sum_{j=1}^q j^p\le\int_0^q(x+1)^p\,dx$. Combining this upper bound with the
necessary condition and rearranging gives
\[
q\ge
\left(1+\frac{(p+1)(1-\alpha-\beta)}{\gamma^T}\right)^{1/(p+1)}-1.
\]
For fixed $p$, $\gamma\in(0,1)$, and $1-\alpha-\beta>0$, this grows
exponentially in $T$.

\smallskip
\noindent\textbf{(ii)}
Lemma~\ref{thm:amplification-search-reliability} gives
$H(I)\le q\varepsilon_{\mathrm{fp}}$, so
$q\varepsilon_{\mathrm{fp}}\le\alpha$ certifies the hallucination constraint.
The same lemma also gives
\[
A(I)\le\prod_{j=1}^q(1-c\rho_j)
\le
\exp\!\left(-c\sum_{j=1}^q\rho_j\right),
\]
using $1-x\le e^{-x}$. Under the assumed geometric lower bound,
\[
\sum_{j=1}^q\rho_j
\ge
\sum_{j=1}^q\min\{1,\rho_1r^{j-1}\}.
\]
Thus $A(I)\le\beta$ is certified whenever
\[
\sum_{j=1}^q\min\{1,\rho_1r^{j-1}\}
\ge c^{-1}\log(1/\beta).
\]
If $\rho_1r^{q-1}\le1$, this sum is
$\rho_1(r^q-1)/(r-1)$, giving the displayed unsaturated sufficient condition and
the equivalent lower bound on $q$ within that regime.

It remains to justify the stated $\Theta(\log(1/\rho_1))$ scaling for the
geometric cumulative-mass condition. Let $L:=c^{-1}\log(1/\beta)>0$ be fixed.
For an upper bound, take
\[
q=\left\lceil\log_r(1/\rho_1)\right\rceil+\left\lceil L\right\rceil+1.
\]
By this time the terms $\min\{1,\rho_1r^{j-1}\}$ have saturated, and there are
at least $\lceil L\rceil$ unit terms, so the cumulative mass is at least $L$.
Thus $q=O(\log(1/\rho_1))$ suffices. For a lower bound, if
$q\le \frac12\log_r(1/\rho_1)$, then all terms are unsaturated for small
$\rho_1$ and
\[
\sum_{j=1}^q\min\{1,\rho_1r^{j-1}\}
\le \frac{\rho_1 r^q}{r-1}
\le \frac{\rho_1^{1/2}}{r-1}
\to 0,
\]
which is eventually smaller than the fixed positive threshold $L$. Hence any
$q$ satisfying the cumulative-mass condition is
$\Omega(\log(1/\rho_1))$, and the scaling is $\Theta(\log(1/\rho_1))$. If
$\rho_1\asymp\gamma^T$, this becomes $\Theta(T)$.
\end{proof}

\subsection{Action Chunking and the Barriers}
\label{app:chunking}

A common intuition is that if the {primitive} action is very low-level (e.g., small joint deltas), then the safe set $\A_{\mathrm{safe}}(s)=\{a:\ f_{\mathrm{phys}}(s,a)=1\}$ is often connected in many states, seemingly avoiding the topological barrier from Section~\ref{sec:topological}. This is frequently true at the \emph{one-step} level, but it does \emph{not} remove multimodality for
long-horizon goal-reaching. Chunking trades \emph{fewer decisions} against a \emph{harder per-decision sampling problem}:
\begin{enumerate}[leftmargin=*,itemsep=2pt,topsep=2pt]
\item \emph{Topology reappears at the progress/chunk level.}
Even if $\A_{\mathrm{safe}}(s)$ is connected for small one-step controls, the \emph{progress} set
$\A_{\mathrm{prog}}(s,t)$ can be disconnected at reachability bottlenecks. Two small safe actions can lead into different time-bounded reachable basins $\Sigma_{t-1}$, while ``in-between'' actions can be safe but \emph{non-progress} (leading to dead ends or timeouts). Chunking amplifies this effect since the progress set typically decomposes into more separated components (distinct partial trajectories/contact-mode prefixes), activating the same ``seam'' phenomenon from Section~\ref{sec:topological}, now with respect to progress/plan validity rather than one-step
physical invalidity.

\item \emph{Precision compounds within a chunk.}
In contact-rich tasks (Section~\ref{sec:codimension}), progress may require staying in a thin tube (or near a
manifold) over multiple successive steps. Requiring consecutive steps in the chunk to remain in such a tube makes the feasible region effectively thinner, decreasing the per-sample mass of $\A_{\mathrm{prog}}(s,t)$ (often sharply) as the chunk length grows.

\item \emph{Horizon compounding improves in count, worsens in mass.}
If the policy outputs chunks of length $\ell$ and commits to executing them, then Lemma~\ref{lem:horizon-barrier} applies with an \emph{effective} horizon of roughly $\lceil T/\ell\rceil$. Increasing $\ell$ reduces the number of factors in this product (helping the horizon barrier) but typically decreases each factor $\gamma_t$ (harder chunk feasibility due to topology/precision and open-loop drift), creating a natural ``sweet spot'' in $\ell$.
\end{enumerate}

This view also clarifies a common empirical design;  predict long chunks for temporal coherence but execute only a short prefix before replanning (receding-horizon), which reduces open-loop compounding and delays irreversible mode commitment while still leveraging temporally structured proposals.

\subsection{Family-level pruning and valid-mass amplification}
\label{app:family-level-pruning}

This appendix subsection formalizes the intuition from the VLA discussion. Fix a round $j$ and write $\rho=Q_j(V)$. Suppose a partial check identifies a measurable family of invalid continuations $B_j\subseteq V^c$, discards this  family, and preserves every valid continuation. Let $Q_j^+$ be the proposal distribution obtained by renormalizing $Q_j$ on $\mathcal P_I\setminus B_j$. Assume $Q_j(B_j)<1$. Then
\[
        Q_j^+(V)
        =\frac{Q_j(V)}{1-Q_j(B_j)}.
\]
If the pruning step removes at least a $\kappa$ fraction of the invalid mass,
namely $Q_j(B_j)\ge \kappa Q_j(V^c)$ for some $\kappa\in[0,1]$, then
\[
        Q_j^+(V)
        \ge
        \frac{Q_j(V)}{1-\kappa(1-Q_j(V))}.
\]
In particular, while $Q_j(V)\le 1/2$,
\[
        Q_j^+(V)
        \ge
        \frac{Q_j(V)}{1-\kappa/2}.
\]
Thus, a pruning operation that repeatedly removes a constant fraction of the currently invalid region can multiply the valid mass by a constant factor until the distribution is no longer dominated by invalid continuations. This is a simple analysis of what hierarchical decomposition, prefix feasibility tests, early collision checks, branch-and-bound, tree search, and replanning are meant to accomplish: they eliminate families of bad futures, rather than merely  rejecting isolated completed samples.

\section{Experimental Setup}
\subsection{Topology Barrier}
\label{app:band-topology-exp}

This appendix describes the implementation and hyperparameters used for the 2D {band topology} experiments in which we measure {action hallucination} and relate it to the isoperimetric lower bound. All experiments were run on a cloud server (4 vCPUs, 15 GB Memory) with two NVIDIA T4 GPUs.

\mypara{Band action-space geometry and safe modes} Actions are two-dimensional, $a=(x,y)\in\mathbb{R}^2$, with an \emph{action band}
\[
x\in[x_{\min},x_{\max}],\qquad y\in[-y_{\max},y_{\max}].
\]
Within this band, the safe set is the union of $M$ disconnected horizontal strips
$\{U_i\}_{i=1}^M$ separated by forbidden gaps. Let $r$ denote the strip half-width and $W$ denote the gap half-width.
The strip centers are placed symmetrically about $0$ with spacing
$2(r+W)$, i.e.
\[
\mu_i \;=\; \Big(i-\frac{M-1}{2}\Big)\,2(r+W),\qquad i\in\{0,\dots,M-1\}.
\]
The safe strips are
\[
U_i \;=\; \bigl\{(x,y): x\in[x_{\min},x_{\max}],\; y\in[\mu_i-r,\mu_i+r]\bigr\}.
\]
Any action outside the band, or in the gaps between strips, is {forbidden}
and counted as a hallucination.

\mypara{Training data distribution} We train on i.i.d.\ safe actions sampled from the union of strips. In the experiments reported here, we perform the following:
\begin{itemize}
  \item sample mode index $i \sim \text{Categorical}(\pi)$ with $\pi$ uniform over modes,
  \item sample $x \sim \text{Uniform}[x_{\min},x_{\max}]$,
  \item sample $y \sim \text{Uniform}[\mu_i-r,\mu_i+r]$.
\end{itemize}

\mypara{Models} All methods produce actions deterministically from a latent/noise variable in
$\mathbb{R}^2$; throughout, the base latent distribution is $Z\sim\mathcal{N}(0,I_2)$. We implemented two models:
\begin{itemize}
\item \textbf{Flow Matching.} We train a rectified flow vector field $v_\theta(x,t)$ with $t\in[0,1]$ using
flow-matching:
\[
x_t = (1-t)x_0 + t x_1,\qquad x_0 \sim p_{\text{data}},\quad x_1\sim\mathcal{N}(0,I_2),
\]
with target velocity $v^\star(x_t,t)=x_1-x_0$, optimized via MSE
$\|v_\theta(x_t,t)-v^\star(x_t,t)\|_2^2$.
Sampling is deterministic Euler (or Heun) integration from $t=1$ to $t=0$
starting at Gaussian noise.
\item \textbf{Diffusion.} We train a v-pred diffusion model with cosine schedule (default $T=200$) and exponential moving average (EMA) of parameters. Training uses MSE on the $v$-prediction target. Sampling is deterministic DDIM (i.e.\ $\eta=0$) with a user-specified number of sampling steps.
\end{itemize}

Both the flow and diffusion networks use the same backbone: a sinusoidal time embedding of dimension $64$, concatenated with the 2D input, followed by an MLP (depth 4 with hidden layers of width 256),
with \texttt{LayerNorm} and \texttt{SiLU} activations, and a final linear layer to $\mathbb{R}^2$.
We optimize with AdamW (default learning rate $2\times 10^{-4}$, weight decay $10^{-4}$), batch size $2048$, and gradient-norm clipping at $1.0$.
Unless otherwise stated, we train each  model for $100,000$ steps.

\mypara{Evaluation protocol and metrics}
For each configuration (defined by $(M,W)$ and the ``smoothness'' settings),
and for each random seed, we generate $N=10^6$ samples in batches of $4096$.
For each method we compute the (i) Hallucination probability, $ H \;=\; \Pr[a(Z)\notin \cup_{i=1}^M U_i]$, where any sample outside the band or in the gaps counts as forbidden, and (ii) the mode masses: $p_i = \Pr[a(Z)\in U_i]$ for $i\in\{1,\dots,M\}$.
We estimate a distribution of {local} Lipschitz proxies on the typical latent ball $B_R=\{z:\|z\|\le R\}$ with $R=3.0$ by sampling 2048 points from $Z\sim\mathcal{N}(0,I_2)$ truncated to
$B_R$, and estimating $\sigma_{\max}(z)\approx \|J(z)\|_{\mathrm{op}}$ via finite differences with step size 0.01.

\subsection{Precision Barrier}
\label{app:prec_exps}

We study a controlled ``glass/mug grasp'' proxy in which valid actions lie on a low-dimensional manifold
$\mathcal{M} \subset \mathbb{R}^d$.
Similar to the topological experiments, we train two generative models (Flow Matching and Diffusion, as in the topology experiments) on the same synthetic manifold distribution and compare them under consistent evaluation. Each action is represented by a 7D vector
\begin{equation}
a = [x, y, h, \mathrm{roll}, \mathrm{pitch}, \sin(\mathrm{yaw}), \cos(\mathrm{yaw})] \in \mathbb{R}^7,
\end{equation}
where yaw is encoded as $(\sin,\cos)$ to avoid discontinuities at $\pm\pi$.
Coordinates are normalized so that a Euclidean distance meaningfully mixes translational and rotational components. 

\mypara{Ideal side-grasp manifold}
The ideal grasp manifold $\mathcal{M}$ is parameterized by two intrinsic coordinates:
\begin{itemize}
    \item $\theta \in [0,2\pi)$: angle around the mug,
    \item $h$ in a valid height interval $[h_{\min}, h_{\max}]$.
\end{itemize}
This defines a $k=2$ dimensional manifold embedded in an ambient space of dimension $d=7$, hence codimension
$d-k = 5$.
Points on $\mathcal{M}$ satisfy: (i) $(x,y)$ lies on a ring of fixed radius $r$ (mug radius plus clearance), (ii) roll and pitch are zero, (iii) yaw points inward toward the mug center (a deterministic function of $\theta$), and (iv) $h$ lies in $[h_{\min}, h_{\max}]$. For any $a \in \mathbb{R}^7$, we compute
\begin{equation}
\mathrm{dist}(a,\mathcal{M}) := \|a - \Pi_{\mathcal{M}}(a)\|_2,
\end{equation}
where $\Pi_{\mathcal{M}}$ is an analytic projection that solves for the best $\theta^*$ (by jointly aligning the $(x,y)$ ring constraint and yaw $(\sin,\cos)$ consistency), clamps $h$ to $[h_{\min}, h_{\max}]$, and sets roll/pitch to $0$.

\mypara{Evaluation protocol and metrics}
For a trained model, we generate i.i.d.\ samples $\{x^{(i)}\}_{i=1}^{N}$ and evaluate the action hallucination curve
\begin{equation}
H(\delta) = \mathbb{P}\big(\mathrm{dist}(x,\mathcal{M}) > \delta\big)
\;\approx\;
\frac{1}{N}\sum_{i=1}^N \mathbf{1}\big[\mathrm{dist}(a^{(i)},\mathcal{M}) > \delta\big],
\end{equation}
over a log-spaced grid $\delta \in [\delta_{\min}, \delta_{\max}]$.
To probe whether concentration near a lower-dimensional manifold requires ill-conditioned transformations, we analyze Jacobians of sampler maps.
Let a sampler consist of step maps
$x_{k+1} = F_k(x_k)$, for $k=0,\dots,K-1$,
so the overall map is $G = F_{K-1}\circ \cdots \circ F_0$.
We compute per-step Jacobians $J_k=\partial F_k/\partial x$ and the global Jacobian
\begin{equation}
J_{\mathrm{global}} = \frac{\partial G}{\partial x} = J_{K-1}\cdots J_0.
\end{equation}
Conditioning is summarized via singular values $\sigma_{\min}(J)$ 
and diagnostics are computed along random sampler trajectories initialized from Gaussian noise.

\section*{Author Contributions and LLM Usage}
H. Soh conceived the central idea and wrote the majority of the manuscript. E. Lim verified the proofs and contributed improvements to the technical statements. LLMs were used to assist with idea exploration, proof development, code implementation, and manuscript editing. Illustrative images depicting the robot in various environments were generated using Gemini, while all schematic figures were created by the authors. All plots were produced from experiments described in the paper and appendix. The authors take full responsibility for the content of this work.

\end{document}